\documentclass[10pt,journal,letterpaper,compsoc]{IEEEtran}
\pdfoutput=1
\ifCLASSOPTIONcompsoc
\else
\fi

\ifCLASSINFOpdf
\else
\fi

\usepackage{url}
\usepackage[numbers]{natbib}
\usepackage{algorithm}
\usepackage{algorithmic}
\usepackage{javen}
\usepackage{graphicx}
\usepackage{subfigure}
\usepackage{colortbl}

\definecolor{shadowww}{gray}{0.4}
\definecolor{shadoww}{gray}{0.7}
\definecolor{shadow}{gray}{0.7}

\begin{document}
%
\title{A Hybrid Loss for Multiclass \\and Structured Prediction}
%
%
%
%

\author{Qinfeng~Shi,
        ~Mark~Reid,
        ~Tiberio~Caetano,
        ~Anton~van den~Hengel
        ~and~Zhenhua Wang
\IEEEcompsocitemizethanks{\IEEEcompsocthanksitem Q. Shi, A. Hengel and Z. Wang are
with The Australian Centre for Visual Technologies and The Computer Vision group of The University of Adelaide, SA, Australia.
 \protect\\
E-mail: \{javen.shi, anton.vandenhengel,zhenhua.wang01\}@adelaide.edu.au}
\IEEEcompsocitemizethanks{\IEEEcompsocthanksitem M. Reid and T. Caetano are
with The Australian National University and NICTA, ACT, Australia. \protect\\
E-mail:mark.reid@anu.edu.au, Tiberio.Caetano@nicta.com.au}
\thanks{}}

%
%

\markboth{IEEE Transactions on Pattern Analysis \& Machine Intelligence, Final draft, Feb.~2014}%
{Shell \MakeLowercase{\textit{et al.}}: Draft only}
%



\IEEEcompsoctitleabstractindextext{%
\begin{abstract}
We propose a novel hybrid loss for multiclass and structured prediction
problems that is a convex combination of a log loss for Conditional
Random Fields (CRFs) and a multiclass hinge loss for Support Vector
Machines (SVMs). We provide a sufficient condition for when the
hybrid loss is Fisher consistent for classification. This condition 
depends on a measure of dominance between labels -- specifically, 
the gap between the probabilities of the best 
label and the second best label.
We also prove Fisher consistency is necessary for parametric consistency
when learning models such as CRFs. We demonstrate empirically that the hybrid loss typically performs
least as well as -- and often better than -- both of its constituent losses on a
variety of tasks, such as human action recognition. 
In doing so we also provide an empirical comparison of the
efficacy of probabilistic and margin based approaches to multiclass
and structured prediction.
\end{abstract}

}

\maketitle

\IEEEdisplaynotcompsoctitleabstractindextext

%
\IEEEpeerreviewmaketitle

\section{Introduction}
\IEEEPARstart{C}onditional Random Fields (CRFs) and Support
Vector Machines (SVMs) can be seen as representative of two
different approaches to classification problems. The former is
purely \emph{probabilistic} -- the conditional probability of
classes given each observation is explicitly modelled -- while the
latter is not -- classification is
performed without any attempt to model probabilities.
Both approaches have their strengths and weaknesses. CRFs~\cite{LafMcCPer01,ShaPer03} 
are known to yield the Bayes optimal solution asymptomatically but do not have known tight generalisation bounds. In
contrast, SVMs have tighter generalisation bounds which typically shrink as the margin grows, and can easily incorporate interesting label-cost such as F1 score or hamming distance in structured cases. But SVMs could be inconsistent when there are
more than two classes~\cite{TewBar07,Liu07}.

Despite their differences, CRFs and SVMs appear very
similar when viewed as optimisation problems. The most salient
difference is the loss used by each: CRFs are trained using a log
loss while SVMs typically use a hinge loss.
In an attempt to capitalise on their relative strengths and avoid
their weaknesses, we propose a novel \emph{hybrid loss} which ``blends'' the two losses. 
After some background (\S\ref{sec:background}) we provide the following analysis:
We argue that Fisher Consistency
for Classification (FCC) -- a.k.a.\ classification calibration -- 
is too coarse a notion and introduce a
distribution-dependent refinement called Conditional Fisher Consistency
for Classification (\S\ref{sec:consistency}).
We prove the hybrid loss is conditionally FCC and give a noise condition that relates 
the hybrid loss's mixture parameter to a margin-like property of the data distribution 
(\S\ref{sec:cfcc}).
We then show that, although FCC is effectively a non-parametric condition, it is
also a necessary condition for consistent risk minimisation using parametric models 
(\S\ref{sec:parametric-fcc}).
Finally, we empirically test the hybrid loss on 
various domains including multiclass classification, text chunking, human action recognition and show it consistently performs as least as well as -- and often better than -- both of its constituent losses (\S\ref{sec:experiments}).

\section{Losses}
\label{sec:background} In classification problems
\emph{observations} $x\in\Xcal$ are paired with \emph{labels}
$y\in\Ycal$ via some joint distribution $D$ over $\Xcal\times\Ycal$.
We will write $D(x,y)$ for the joint probability and $D(y|x)$ for
the conditional probability of $y$ given $x$. Since the labels $y$
are finite and discrete we will also use the notation $D_y(x)$ for
the conditional probability to emphasise that distributions over
$\Ycal$ can be thought of as vectors in $\RR^k$ for $k=|\Ycal|$.
We will also use $p$ and $q$ to denote distributions over $\Ycal$
but reserve their use for distributions generated by models. 

\subsection{Multiclass Prediction}
\label{sec:multiclass}
When the number of possible labels is $k=|\Ycal|>2$ the 
classification problem is known as a \emph{multiclass} classification problem. 

Given $m$ training observations $S=\{(x_i,y_i)\}_{i=1}^m$ drawn i.i.d.\ 
from $D$, the aim of the learner is to produce a predictor $h : \Xcal
\to \Ycal$ that minimises the \emph{misclassification error} $e_D(h)
= \PP_D\left[ h(x) \ne y \right]$. Since the true distribution is
unknown, an approximate solution to this problem is typically found
by minimising a regularised empirical estimate of the risk for a
\emph{surrogate loss} $\ell$. Examples of surrogate losses will be
discussed below.

Once a loss is specified, a solution is found by solving
\begin{equation}\label{eq:erm}
    \min_{f} \frac{1}{m}\sum_{i=1}^m \ell(f(x_i),y_i) + \Omega(f)
\end{equation}
where each \emph{model} $f : \Xcal \to \RR^k$ assigns a vector of
scores $f(x)$ to each observation and the regulariser $\Omega(f)$
penalises overly complex functions. A model $f$ found in this way
can be transformed into a predictor by defining $h_f(x) =
\argmax_{y\in\Ycal} f_y(x)$ where ties are broken in some 
arbitrary but deterministic way (see Section~\ref{sec:parametric-fcc} for 
details). 
We will overload the definition of 
misclassification error and sometimes write $e_D(f)$ as shorthand
for $e_D(h_f)$.

A common surrogate loss for multiclass problems is a generalisation
of the binary class hinge loss used for SVMs~\cite{CraSin00}:
\begin{equation}\label{eq:hinge}
    \ell_H(f,y) = \left[1 - M(f,y)\right]_+
\end{equation}
where $[z]_+ = z$ for $z>0$ and is 0 otherwise, and $M(f,y) = f_y -
\max_{y'\ne y}f_{y'}$ is the \emph{margin} for the vector $f\in\RR^k$.
Intuitively, the hinge loss is minimised by models that not only
classify observations correctly but also maximise the difference
between the highest and second highest scores assigned to the
labels. 

While there are other, consistent losses for SVMs~\cite{TewBar07,Liu07}, these cannot scale up to a large $k$. For example, the multiclass hinge loss $\sum_{j\neq y}[1+f_j(x)]_+$ is shown to be 
consistent in \cite{Liu07}. However, it requires evaluating $f$ on all possible labels except the true $y$. This is intractable for labels where the possible assignments grow exponentially. 
The other known and consistent multiclass hinge losses have similar intractability.


\subsection{Structured Prediction}
\label{sub:multiclass}
In the multiclass case $\{(x_i,y_i)\}_{i=1}^m$ are assumed i.i.d. However, in many cases, they are not i.i.d. Structured prediction \cite{BakHofSchSmoetal07} can deal with these cases by grouping correlated labels to form a structured label $y$. Here the structured label $y$ can be any object associated with $x$. For example, for the automated paragraph breaking problem, the input $x$ is a document, and the output $y$ is a sequence whose entries denote the beginning positions of the paragraphs. For image segmentation, the input $x$ is an $n_1$ by $n_2$ image, and the structured label $y$ is a 2-D lattice $\{y^{i,j}\}_{i = 1, \cdots, n_1;1= 1, \cdots, n_2}$. The framework of Probabilistic Graphical Models (PGMs) \cite{koller2009probabilistic} provides a principled way of modelling the dependencies of the components of $y$. For a $y$ with $L$ components \ie $y = (y^1,y^2, \cdots, y^L)$,  the graph of PGM $\Gcal = (\Vcal,\Ecal)$ consists of the node set $\Vcal = \{1,\cdots, L\}$ and the edge set $\Ecal$ that reflects the dependencies. Assuming each component $y^j \in \{1,\cdots,c\}$ for all $j$, there are $k = c^L$ many possible assignments for $y$. In other words, such a structured label $y$ can be seen as a multiclass problem with $k$ many classes in theory, although many multiclass algorithms will be intractable in the structured case. 

Once the structured labels are formed, we can assume the structured input-output pairs $\{(x_i,y_i)\}_{i=1}^m$ are i.i.d. from some joint distribution. The predictors $h$ or the models $f$ can be learned in a fashion similar to \eq{eq:erm}. 
The models are usually specified in terms
of a parameter vector $w\in\RR^n$ and a feature map $\phi :
\Xcal\times\Ycal \to \RR^n$ by defining $f_y(x;w) =
\inner{w}{\phi(x,y)}$ and in this case the regulariser is
$\Omega(f) = \frac{\lambda}{2}\|w\|^2$ for some choice of
$\lambda\in\RR$. This is the framework used to implement the SVMs
and CRFs used in the experiments described in Section~\ref{sec:experiments}.
Although much of our analysis does not assume any particular parametric model,
we explicitly discuss the implications of doing so in \S\ref{sec:parametric-fcc}.

\subsection{Probabilistic Interpretation and the Hybrid Loss}
\label{sec:background-pmodel} 
CRFs are based on a model whereby 
\begin{equation}\label{eq:prob}
    p_y(x;f) = \frac{\exp(f_y(x))}{\sum_{y\in\Ycal}{\exp(f_y(x))}},
\end{equation}
and use the log loss
\begin{equation*}
    \ell_L(p,y) = -\ln p_y.
\end{equation*}
This loss penalises models that assign low probability to likely labels and, implicitly, that assign high probability to
unlikely labels.

We can see that \eq{eq:prob} provides a probabilistic interpretation of the scores of $f_y(x)$.
It is easy to show that under this interpretation the hinge loss for $p = p(\cdot;f)$ is given by
\begin{equation*}
    \ell_H(p,y) = \left[1-\ln\frac{p_y}{\max_{y'\ne y}p_{y'}}\right]_+
\end{equation*}

We now propose a novel \emph{hybrid loss} that is a 
combination of the hinge and log losses
\begin{equation}
    \ell_\alpha(p,y)
    = \alpha\ell_L(p,y) + (1-\alpha)\ell_H(p,y) \label{eq:hybrid}
\end{equation}
where the mixture of the two losses is controlled by a parameter
$\alpha\in[0,1]$. Setting $\alpha = 1$ or $\alpha = 0$ recovers the
log loss or hinge loss, respectively. The intention is that choosing
$\alpha$ close to 0 will emphasise having the maximum gap between the largest and 
second largest label probabilities while an $\alpha$ close to 1 will
force models to prefer accurate probability assessments over strong
classification.

This family of hybrid losses is similar to a recent proposal 
by Zhang et al \cite{Zhang:2009}. They also define a single parameter family of loss functions
called \emph{coherence functions} that interpolate between hinge loss and a loss that
is closely related to loss based on log-likelihood. Like the loss presented here, their
losses are surrogates for 0-1 loss and both families have the hinge loss as a limit point.
A key difference between the two proposals has to do with the consistency of losses in
each family: the coherence losses are all Fisher consistent for 
probability estimation whereas the hybrid losses satisfy a weaker form of consistency
which we call conditional Fisher Consistency for Classification and analyse below.

Despite of the properties of the coherence functions, using them in structured cases is intractable. They require the evaluation of a function $\beta_j(x)$ for all classes \ie $  j = 1,\cdots, k$, see Algorithm 1 step 2 (c) in \cite{Zhang:2009}. Note that $k$ grows exponentially in structured cases. They encounter the same problem as other consistent multiclass SVMs. Our hybrid loss does not have this problem.

\section{Fisher Consistency For Classification}\label{sec:consistency}
A desirable property for a loss is that, given enough data, the models
obtained by minimising the loss at each observation will make predictions
that are consistent with the true label probabilities at each observation. We are mainly concerned with distributions 
$D(x)$ over the set $\Ycal$ for some fixed (but irrelevant) $x$. 
We will therefore overload $D$ and use it to denote a distribution 
over $\Ycal$. Whether $D$ represents a distribution over labels or a distribution
over labels and observations should be clear from context.

We say a vector $f\in\RR^{\Ycal}$ is \emph{aligned} with
a distribution $D$ over $\Ycal$ whenever maximisers of $f$ are
also maximisers for $D$. That is, when
\(
    \argmax_{y\in\Ycal} f_y \subseteq \argmax_{y\in\Ycal} D_y.
\)
Since the probabilistic models described in \S\ref{sec:background-pmodel} 
pass the components of a vector $f(x)$ through $\exp$ and rescale, it is clear
that a prediction $f(x)$ is aligned with $D$ if and only if $p_y(x;f)$ is
aligned with $D$.
Because of this correspondence, the following definitions of consistency are
equivalent regardless of whether general models and losses or their 
probabilistic counterparts are used.

If, for all label distributions $D$, minimising the conditional risk 
$L(f,D) = \EE_{y\sim D}[\ell(f,y)]$ for a loss $\ell$ yields a vector $f^*$ aligned
with $D$ we will say $\ell$ is \emph{Fisher consistent for
classification} (FCC)~\footnote{%
    Note that the Fisher consistency for classification is weaker than Fisher 
    consistency for density estimation. The former requires the same prediction 
    only, 
    while the latter requires the estimated density is the same as the true data
    distribution. In this paper, we focus on the former only.
    For an analysis of Fisher consistency for density estimation, we refer
    the reader to \cite{Reid:2009b}.} 
-- or \emph{classification calibrated} \citep{TewBar07}.
This is an important property for losses since it is
related to the asymptotic consistency of the empirical risk
minimiser for that loss \citep[Theorem 2]{TewBar07}.

The standard multiclass hinge loss $\ell_H$ is known to be inconsistent for
classification when there are more than two
classes~\citep{Liu07,TewBar07}. The analysis in \citep{Liu07} shows
that the hinge loss is inconsistent whenever there is an instance $x$ with
a \emph{non-dominant} distribution -- that is, $D_y(x) <
1/2$ for all $y\in\Ycal$. Conversely, a distribution is
\emph{dominant} for an instance $x$ if there is some $y$ with
$D_y(x) \ge 1/2$.
In contrast, the log loss used to train non-parametric CRFs is
Fisher consistent for probability estimation -- that is, the
associated risk is minimised by the true conditional distribution --
and thus $\ell_C$ is FCC since the minimising distribution is equal
to $D(x)$ and thus aligned with $D(x)$. 

\subsection{Conditional Consistency of the Hybrid Loss}
\label{sec:cfcc}

In order to analyse the consistency of the hybrid loss we require the 
following refined notion of Fisher consistency. 
If $D=(D_1,\ldots,D_k)$ is a
(conditional) distribution over the labels $\Ycal$ then we say the loss 
$\ell$ is \emph{conditionally FCC with respect to $D$} whenever minimising the
conditional risk w.r.t. $D$, $L(f,D) = \EE_{y\sim
D}\left[\ell(f,y)\right]$ yields a predictor $f^*$ that is
aligned with $D$. Of course, if a loss $\ell$ is conditionally
FCC w.r.t. $D$ for \emph{all} $D$ it is, by definition,
(unconditionally) FCC.

The following theorem provides sufficient conditions on the hybrid parameter
$\alpha$ in terms of a label distribution $D$ so that the hybrid loss 
$\ell_\alpha$ is conditionally FCC w.r.t. $D$.
\begin{theorem}\label{thm:fcc}
    Let $D=(D_1,\ldots,D_k)$ be a distribution over $\Ycal$ and let
    $D_{max} := \max_y D_y$ be the largest probability assigned to any 
    $y\in\Ycal$. Also let $\Ycal_{max} := \{ y : D_y = D_{max} \}$ be the set of 
    labels with maximal probability 
    and $D_{next} := \max_{y\notin \Ycal_{max}} D_y$ be
	the second largest probability assigned to a label, or $D_{next} = \infty$
	if $\Ycal_{max} = \Ycal$.
    Then the hybrid loss $\ell_\alpha$ is conditionally FCC for $D$ whenever
    $D_{max} \ge \frac{1}{2}$ or
    \begin{equation}\label{eq:alpha}
        \alpha > 1 - \frac{D_{max} - D_{next}}{1-2D_{max}}.
    \end{equation}
\end{theorem}

The proof is by contradiction and proceeds at a high level by showing that if the
distribution $D$ satisfies $D_{max} \ge \frac{1}{2}$ or \eqref{eq:alpha} but the
minimiser $p$ of the risk $L_\alpha(p,D)$ is not aligned with $D$ we derive a 
falsehood.
The argument is broken into two cases: when the risk minimising distribution $p$ 
has a unique maximum probability and when it does not. In both cases we show how to
construct an alternative distribution $q$ (that depends on $D$) such that
$L_\alpha(q,D) < L(p,D)$, yielding the required contradiction.
In the first case, $q$ is obtained by swapping the most probable label of $p$ with
that of $D$. In the second case (when $p$ has two or more most probable labels),
$q$ is obtained by perturbing $p$ slightly towards $D$.

\begin{proof}
    Since we are free to permute the labels within $\Ycal$, we will assume without loss
    of generality that there are $t$ ties for the most probable label and that 
    $\Ycal_{max} = \{1, \ldots, t\}$ and so $D_1=\cdots=D_t$.
    Defining $L_\alpha(p,D) = \EE_{y\sim D}\left[\ell_\alpha(p,y)\right]$,    
    the proof now proceeds by contradiction by assuming that there is some minimiser
    $p = \argmin_q L_\alpha(q,D)$ that is not aligned with
    $D$.
    For this to occur there must be some label $y^*>t$ such that $p_{y^*}$ 
    is as least as large as $p_1, \ldots, p_t$. 
    For simplicity, and again without loss of generality, we will assume that
    $y^*$ is the label with the largest probability according 
    to $p$ (that is, $y^* \in \argmax_y p_y$).
    We are also free to have permuted labels within $\Ycal_{max}$ to ensure 
    $t \in \argmax_{y\in\Ycal_{max}} p_y$.

    The first case to consider is when $p_{y^*}$ is strictly larger than
    $p_t$.
    Here we construct a new distribution $q$ that swaps the 
    values of $p_t$ and $p_{y^*}$ and leaves all the other values unchanged.
    That is, $q_{t} = p_{y^*}$, $q_{y^*} = p_t$ and $q_y = p_y$ for all 
    $y\in\Ycal-\{t,y^*\}$.
    Intuitively, we will now show that this new point is ``closer'' to $D$ and 
    therefore the CRF component of the loss will be reduced while the SVM component of the loss 
    won't increase.
    To do so, we consider the difference in conditional risks:
    \begin{eqnarray}
        L_\alpha(p,D) - L_\alpha(q,D)
        &=& \sum_{y=1}^k D_y.(\ell_\alpha(p,y) - \ell_\alpha(q,y)) \nonumber \\
        &=& D_{t}.(\ell_\alpha(p,t) - \ell_\alpha(q,t)) \nonumber\\
        &&  + D_{y^*}.(\ell_\alpha(p,y^*) - \ell_\alpha(q,y^*)) \nonumber \\
        &=& (D_{t} - D_{y^*})(\ell_\alpha(q,y^*) - \ell_\alpha(q,t)) \nonumber
    \end{eqnarray}
    since $\ell_\alpha(p,t)=\ell_\alpha(q,y^*)$ and
    $\ell_\alpha(p,y^*) = \ell_\alpha(q,t)$ and the other terms cancel by
    construction.
    Since, by assumption $y^* > t$, we have $D_t - D_{y^*} > 0$, so all that 
    is required now is to show that
    $\ell_\alpha(q,y^*) - \ell_\alpha(q,t) = \alpha\ln\frac{q_t}{q_{y^*}}
    +(1-\alpha)(\ell_H(q,y^*)-\ell_H(q,t))$
    is strictly positive.

    Since $p_{y^*} = q_t > q_y$ for $y\ne t$ we have $\ln\frac{q_t}{q_{y^*}} > 0$,
    $\ell_H(q,y^*) = \left[1-\ln\frac{q_{y^*}}{q_t}\right]_+ > 1$, and
    $\ell_H(q,t) = \left[1-\ln\frac{q_t}{q_{y^*}}\right]_+ < 1$, and so
    $\ell_H(q,y^*) - \ell_H(q,t) > 1 - 1 = 0$.
    Thus, $\ell_\alpha(q,y^*) - \ell_\alpha(q,t) > 0$ as required.
    This gives a contradiction and thus establishes the theorem in the case where
    $p_{y^*} > p_t$.

    Now suppose that $p_{y^*} = p_t$. That is, there is a tie for the maximum
    probability label in $p$ and at least one of these maximising labels 
    coincides with the maximising labels of $D$.
    In this case we show that a slight perturbation of $p$ yields a new distribution
    $q$ with a strictly smaller loss.
    To define $q$ we let $\epsilon>0$ and set 
    $q_{y^*} = p_{y^*} + \epsilon$, $q_t = p_t - \epsilon$, and $q_y = p_y$ for all
    other $y \ne t, y^*$.
    Now, for $y\ne t, y^*$ we have $\ell_L(p,y) - \ell_L(q,y) = 0$. 
    Also for $y\ne t, y^*$ we have $p_t > p_y$ and $q_{y^*} > q_y$ thus
    $\ell_H(p,y) - \ell_H(q,y)
    = 1-\ln\frac{p_y}{p_t} - \left(1 -\ln\frac{q_y}{q_{y^*}}\right)
    = \ln\frac{p_t}{q_{y^*}}
    > 1-\frac{q_{y^*}}{p_t}
    = -\frac{\epsilon}{p_t}$
    since $-\ln x > 1 - x$ for $x > 0$ and
    $q_{y^*} = p_{y^*} + \epsilon = p_t + \epsilon$.
    By substituting this inequality into the definition of $\ell_\alpha$, we see
    that for all $y \ne t,y^*$ 
    \begin{equation}\label{eq:yne12}
        \ell_\alpha(p,y) - \ell_\alpha(q,y) > -\epsilon\frac{(1-\alpha)}{p_1}.
    \end{equation}

    For the label $y^*$ we see that the log loss component of $\ell_\alpha$ 
    satisfies
    $\ell_L(p,y^*) - \ell_L(q,y^*) = -\ln\frac{p_{y^*}}{q_{y^*}} 
    > \frac{q_{y^*} - p_{y^*}}{q_{y^*}} = \frac{\epsilon}{q_{y^*}}$ and the 
    difference between the hinge loss components becomes
    $\ell_H(p,y^*)-\ell_H(q,y^*) 
    = (1-\ln\frac{p_{y^*}}{p_t}) - (1-\ln\frac{q_{y^*}}{q_t}) =
    \ln\frac{q_{y^*}}{q_t} = \ln\frac{p_{y^*}+\epsilon}{p_{y^*}-\epsilon}$ since
    $p_{y^*} = p_t$.
    Thus $\ell_H(p,y^*)-\ell_H(q,y^*) 
    > 1 - \frac{p_{y^*} - \epsilon}{p_{y^*} + \epsilon}
    = \frac{2\epsilon}{p_{y^*}+\epsilon}$.
    And so
    \begin{align}
        \ell_\alpha(p,y^*) - \ell_\alpha(q,y^*)
        & > \epsilon\left[\frac{\alpha}{p_{y^*}+\epsilon}
            + \frac{2(1-\alpha)}{p_{y^*}+\epsilon}\right] \notag\\
        &= \epsilon\frac{2-\alpha}{p_{y^*}+\epsilon} 
         > \epsilon\frac{2-\alpha}{p_{y^*}} \label{eq:y1}
    \end{align}
    since $\epsilon>0$ and $\alpha\le 1$.
    Finally, for the label $t$ we have
    $\ell_L(p,t)-\ell_L(q,t) = -\ln\frac{p_t}{q_t}
    > \frac{q_t - p_t}{q_t} 
    = \frac{-\epsilon}{q_t} = \frac{-\epsilon}{p_{y^*}-\epsilon}$ since
    $q_t = p_t - \epsilon$ and $p_t = p_{y^*}$. Similarly,
    $\ell_H(p,t) - \ell_H(q,t) = (1-\ln\frac{p_t}{p_t})-(1-\ln\frac{q_t}{q_t})
    = \ln\frac{q_t}{q_{y^*}} > 1 - \frac{q_{y^*}}{q_t} =
    \frac{-2\epsilon}{p_{y^*}-\epsilon}$.
    Thus,
    \begin{align}
        \ell_\alpha(p,t) - \ell_\alpha(q,t)
        &> -\epsilon\left[
            \frac{\alpha}{p_{y^*}-\epsilon} + \frac{2(1-\alpha)}{p_{y^*}-\epsilon}
        \right] \notag\\
        &= -\epsilon\frac{2-\alpha}{p_{y^*}-\epsilon}
         > -\epsilon\frac{2-\alpha}{p_{y^*}} \label{eq:y2}
    \end{align}

    Putting the inequalities (\ref{eq:yne12}), (\ref{eq:y1}) and (\ref{eq:y2})
    together yields
    \begin{eqnarray*}
        \lefteqn{L_\alpha(p,D) - L_\alpha(q,D)} \\
        &>& D_{y^*} \epsilon \left[\frac{2-\alpha}{p_{y^*}}\right]
            - D_t \epsilon \left[\frac{2-\alpha}{p_{y^*}}\right]
            - \sum_{y\ne y^*,t} D_y \epsilon \frac{1-\alpha}{p_{y^*}} \\
        &=& \frac{\epsilon}{p_{y^*}}
            \left[
                (D_{y^*}-D_t)(2-\alpha) - (1-D_{y^*}-D_t)(1-\alpha)
            \right] \\
        &=& \frac{\epsilon}{p_{y^*}}\left[D_{y^*} - D_t + (1-\alpha)(2D_{y^*} - 1)\right].
    \end{eqnarray*}
    Since $D_{y^*} > D_t$,
    when $D_{y^*} \ge \frac{1}{2}$ the final term is non-negative without any 
	additional constraint on $\alpha\in[0,1]$ and since $D_{y^*} > D_t$, the
	difference in risks is thus positive.
    When $D_{y^*} < \frac{1}{2}$ the difference in risks is positive
    whenever
    \begin{equation}
        \alpha > 1 - \frac{D_{y^*} - D_t}{1 - 2D_{y^*}}.
    \end{equation}
    Observing that $D_{max} = D_{y^*}$ and $D_{next} = D_t$
    completes the proof.
\end{proof}

Theorem~\ref{thm:fcc} can be inverted and interpreted as a
constraint on the conditional distributions of some data distribution $D$ such that a 
hybrid loss with parameter $\alpha$ will yield consistent predictions. Specifically,
the hybrid loss will be consistent if, for all $x\in\Xcal$ such that
$q = D(x)$ has no dominant label (\emph{i.e.}, $D_y(x) \le
\frac{1}{2}$ for all $y\in\Ycal$), the gap $D_{y_1}(x) -
D_{y_2}(x)$ between the top two probabilities is larger than
$(1-\alpha)(1-2D_{y_1}(x))$. When this is not the case for some $x$,
the classification problem for that instance is, in some sense, too
difficult to disambiguate. In this sense, the bound can be seen as a property
on distributions akin to Tsybakov's noise condition 
\citep{Che06}. Both conditions are non-constructive as they depend on the 
unknown distribution but provide some guidance as to the effect of parameter
choices (i.e., $\alpha$ for the hybrid loss and regularisation constants for SVMs).
Exploring the relationship between conditional FCC and the Tsybakov noise condition
is the focus of ongoing work.

\subsection{Parametric Consistency}\label{sec:parametric-fcc}

Since Fisher consistency is defined point-wise on observations, it is not directly
applicable to parametric models as these enforce inter-observational constraints 
(\emph{e.g.} smoothness).
Abstractly, assuming parametric hypotheses can be seen as a restriction over
the space of allowable scoring functions. When learning parametric models, risks
are minimised over some subset $\Fcal$ of functions from $\Xcal \to \RR^{\Ycal}$
instead of all possible functions.
We now show that, given some weak assumptions on the hypothesis class $\Fcal$, a 
loss being FCC is a necessary condition if the loss is also to be 
$\Fcal$-consistent.

We say a loss $\ell$ is $\Fcal$-consistent if, for any distribution, minimising its
associated risk over $\Fcal$ yields a classifier with minimal
0-1 loss in $\Fcal$.\footnote{%
    While this is simpler and stronger than the usual asymptotic notation of 
    consistency \cite{LugVay04} it most readily relates to FCC and suffices for our 
    discussion since we are only establishing that FCC is a necessary condition.}
Recall from Section~\ref{sub:multiclass} that the risk of a hypothesis $f\in\Fcal$ associated with a loss $\ell$ and
distribution $D$ over $\Xcal\times\Ycal$ is $L_D(f) = \EE_D\left[\ell(y,f(x))\right]$
and the 0-1 risk or misclassification error for $f$ is
$e_D(f) = \PP_D\left[h(x) \neq y \right]$ where $h$ is some 
classifier deterministically derived by some tie-breaking procedure on $f$.
More precisely, we will say a \emph{tie-breaker} $T$ is a function from the power set
of $\Ycal$ to $\Ycal$ that guarantees $T(Y) \in Y$ for all non-empty $Y \subset \Ycal$. 
Finally, we define $h_f(x) = T(\argmax_{y\in\Ycal} f_y(x))$ to be
the classifier derived from $f$ using $T$. 

Given a function class $\Fcal$ we say \emph{$\ell$ is $\Fcal$-consistent} if, for all
distributions $D$ and all tie-breakers $T$ defining the classifiers $h_f$,
\begin{equation}\label{eq:F-consistency}
    L_D(f^*) = \inf_{f\in\Fcal} L_D(f)
    \implies
	e_D(h_{f^*}) = \inf_{f\in\Fcal} e_D(h_f).
\end{equation}


We need a relatively weak condition on function classes $\Fcal$ to state our theorem.
We say a class $\Fcal$ is \emph{regular} if the follow two properties hold: 1) For any
$g\in\RR^{\Ycal}$ there exists an $x\in\Xcal$ and an $f\in\Fcal$ so that $f(x) = g$;
and 2) For any $x\in\Xcal$ and $y\in\Ycal$ there exists an $f\in\Fcal$ so that there
is a \emph{unique} $y \in \Ycal$ which maximises $f_{y}(x)$.

Intuitively, the first condition says that for any distribution over labels
there must be a function in the class which models it perfectly on some point in the
input space.
The second condition requires that any mode can be modelled on any input by a function
that has no ties for its maximum value.
Importantly, these properties are fairly weak in that they do not say anything about the
constraints a function class might put on relationships between distributions modelled
on different inputs.


\begin{theorem}\label{thm:fcc-necessary}
    For regular function classes $\Fcal$ any loss that is $\Fcal$-consistent is
    necessarily also Fisher Consistent for Classification (FCC).
\end{theorem}

\begin{proof}
    The proof is by contradiction.
    We assume we have a regular function class $\Fcal$ and a loss $\ell$ which is
    $\Fcal$-consistent but not FCC.
    That is, (\ref{eq:F-consistency}) holds for $\ell$ but there exists a
    distribution $p$ over $\Ycal$
    such that there is a $g \in \RR^{\Ycal}$ which minimises the conditional risk
	$L_{p}(g)$ but $g$ is not aligned with $p$ 
	(i.e., $\argmax_{y\in\Ycal} g_y \not\subset \argmax_{y\in\Ycal} p_y$).

    By the assumption of the regularity of $\Fcal$, property 1 means there is 
	an $x\in\Xcal$ and a
    $f\in\Fcal$ so that $f(x) = g$.
    We now define a distribution $D$ over $\Xcal\times\Ycal$ that puts all its mass
    on the set $\{x\}\times\Ycal$ so that $D(x,y) = p_y$.
    Since this distribution is concentrated on a single $x$ its full risk and
    conditional risk on $x$ are the same. That is, $L_D(\cdot) = L_p(\cdot)$.
    Thus,
    \[
        L_D(f) = L_p(f) = \inf_{f'\in\Fcal} L_p(f') = \inf_{f'\in\Fcal} L_D(f')
    \]
    By the assumption of $\Fcal$-consistency, since $f$ is a minimiser of $L_D$ the
	classifier $h_f$ must also minimise $e_D$ for any choice of tie-breaker $T$
	used to define $h_f$.
	Because $g = f(x)$, the construction of $D$ implies that
	$e_D(h_f)   = e_p(h_g) 
	 			= \PP_{y\sim p}\left[y \neq h_{g}(x)\right] 
				= 1 - p_{y_g}$ 
	where $y_g = h_g(x)$ is the label predicted by $h_g$.
	However, since $g$ is not aligned with $p$ by assumption and 
	\eqref{eq:F-consistency} holds for any $T$, we are free to 
	choose the tie-breaker $T$ defining $h_g$ so that 
	$h_g(x) = T(\argmax_{y} g_y(x)) \notin \argmax_y p_y$. 
	Thus
	\begin{equation}\label{eq:eDbigger}
        e_D(h_f) = e_p(h_g) = 1-p_{y_g} > 1 - p_{y^*}
	\end{equation}
    since $y_g \neq y_* \in \argmax_y p_y$.

    By the second regularity property, there must also be an $\hat{f}\in\Fcal$ such
    that $y^*$ is the unique maximiser of $\hat{f}_y(x)$ for $y \in \Ycal$.
	Since $y^*$ is a unique maximiser, any choice of tie-breaker $T$ will result
	in a classifier satisfying $h_{\hat{f}}(x) = y^*$ as any $T$ must guarantee
	$T(\{ y \}) \in \{ y \}$ for all $y \in \Ycal$.
	Therefore, we arrive at the contradiction
	\[
		1 - p_{y^*} = e_D(h_{\hat{f}}) \geq e_D(h_f) > 1 - p_{y^*}
	\]
	since $h_f$ is a minimiser of $e_D$ and $e_D(h_f) > 1 - p_{y^*}$ by 
	\eqref{eq:eDbigger}.
    Thus, we have shown that there exists a distribution $D$ so $f\in\Fcal$ is a
    minimiser of the risk $L_D$ but $h_f$ is not a minimiser of the
    misclassification rate $e_D$, contradicting the assumption of the
    $\Fcal$-consistency of $\ell$. Therefore, $\ell$ must be FCC.

\end{proof}

The above analysis of the hybrid loss suggests it should outperform the hinge loss due to its improved consistency on distributions with 
non-dominant labels. Furthermore, it should also make more efficient use of 
data than log loss on distributions with dominant labels.
These hypotheses are confirmed in the next section by applying the hybrid, log and 
hinge losses to a 
number of synthetic multiclass data sets in which the data set size and proportion 
of examples with non-dominant labels are carefully controlled. 

We also compare the hybrid loss with the log and hinge losses on
several real structured estimation problems and observe that the hybrid loss
regularly outperforms the other losses and consistently performs at least as well
as either of the other losses on any problem.

\section{Multiclass Classification}
Two types of multiclass simulations were performed. The first examined the 
performances of the hybrid, log and hinge losses when no observation had a 
dominant label. That is all observations were drawn from a $D$ with $D_y(x) <1/2$ 
for all labels $y$. The second experiment considered distributions with a
controlled mixture of observations with dominant and non-dominant labels.

\subsection{Non-dominant Distributions} 

\begin{figure}[tb]
 \centering
     \includegraphics[width=0.9\linewidth]{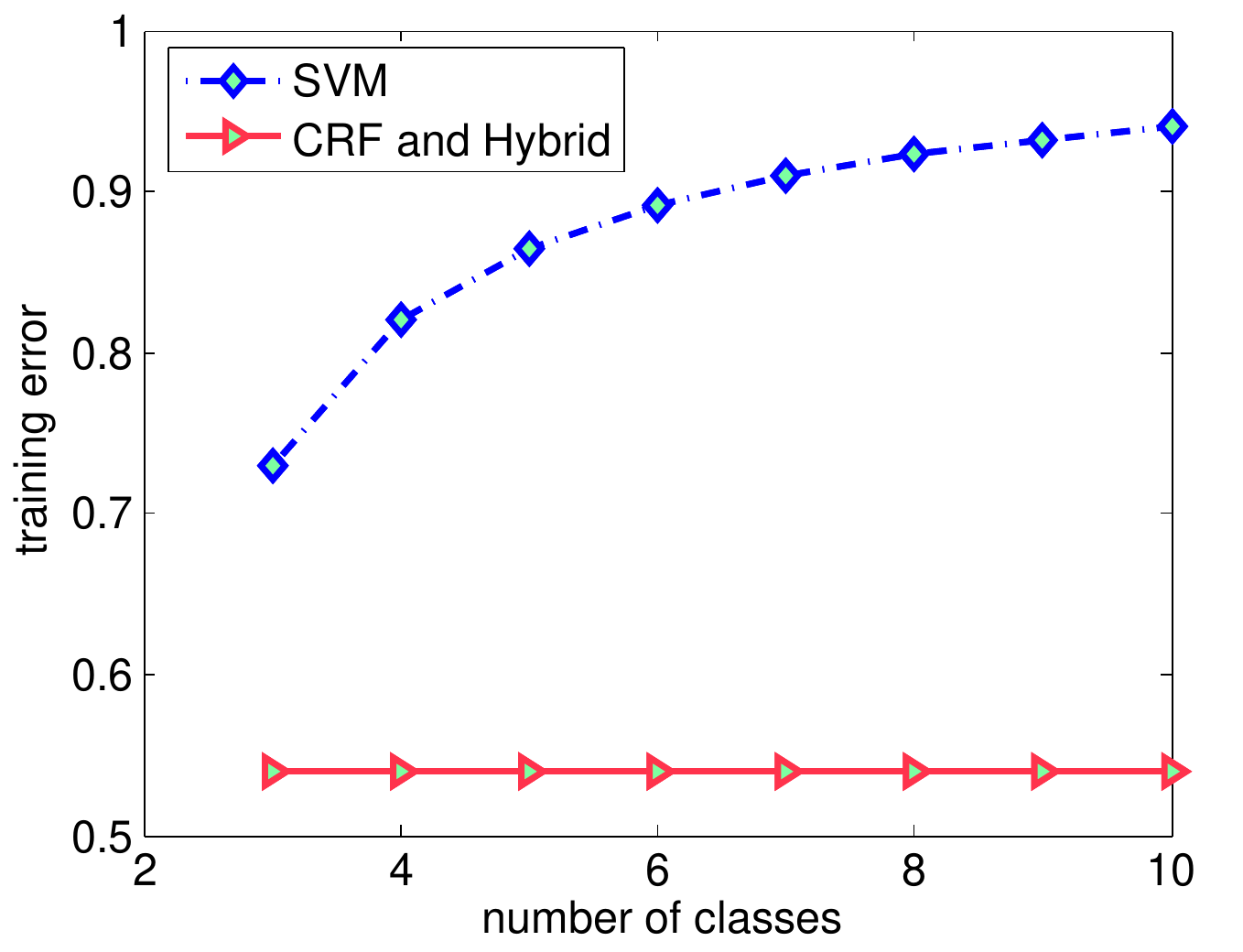}
  \caption{Training Error with various number of classes. $\alpha =
  0.5$ for the hybrid loss. Fisher consistency analyses the behaviour of a loss observing the entire data population. Thus the training data are the entire data, so are the testing data. Consequently, the training error is the testing error.
} \label{fig:multiclass}
\end{figure}

\begin{figure*}[tb]
 \centering
 \subfigure[Hybrid v.s. Hinge (31/15)]{
     \includegraphics[width=0.3\linewidth]{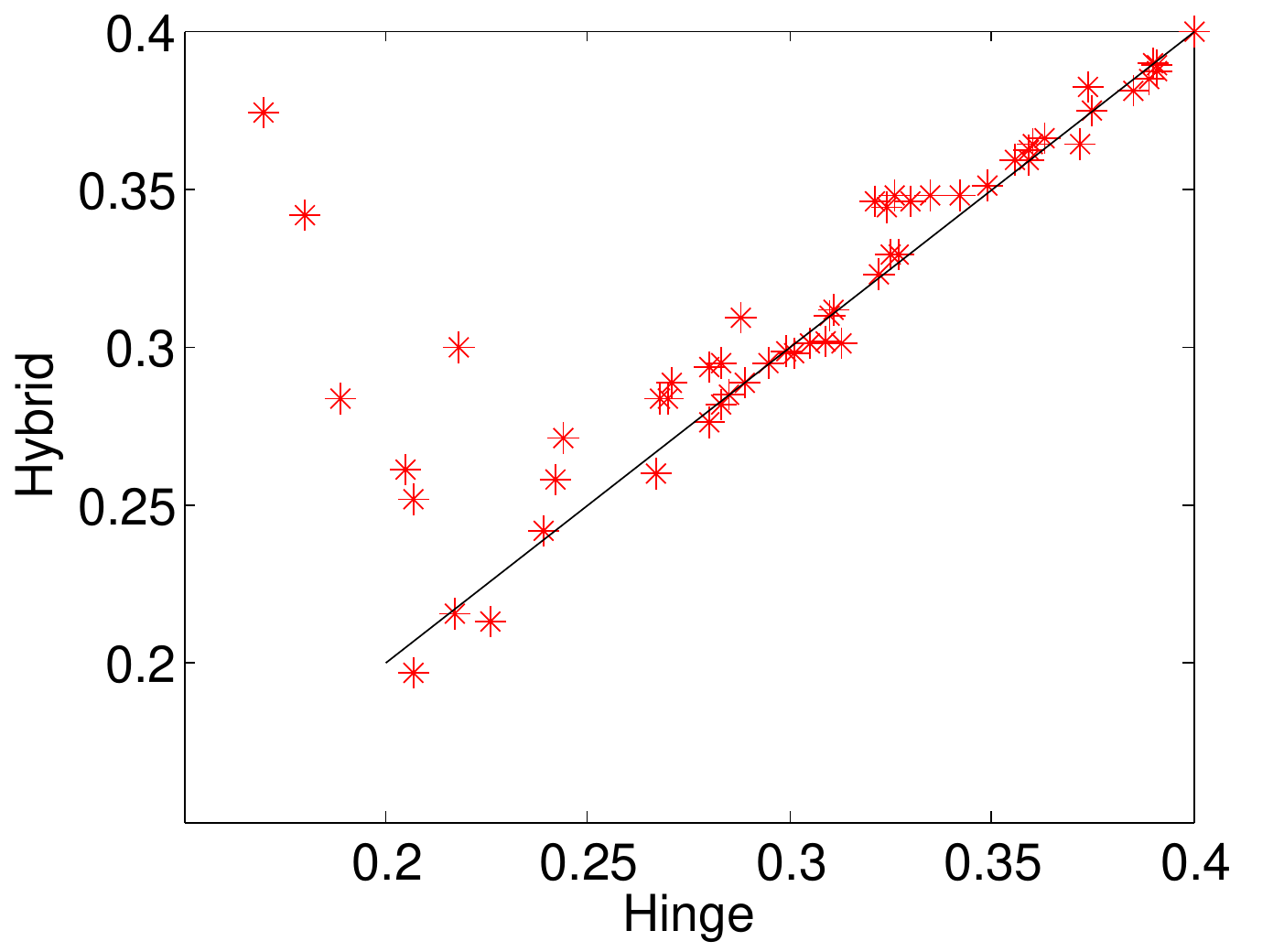}
     \label{fig:hybrid_svm}
     }
      \subfigure[Hybrid v.s. Log (34/15)]{
     \includegraphics[width=0.3\linewidth]{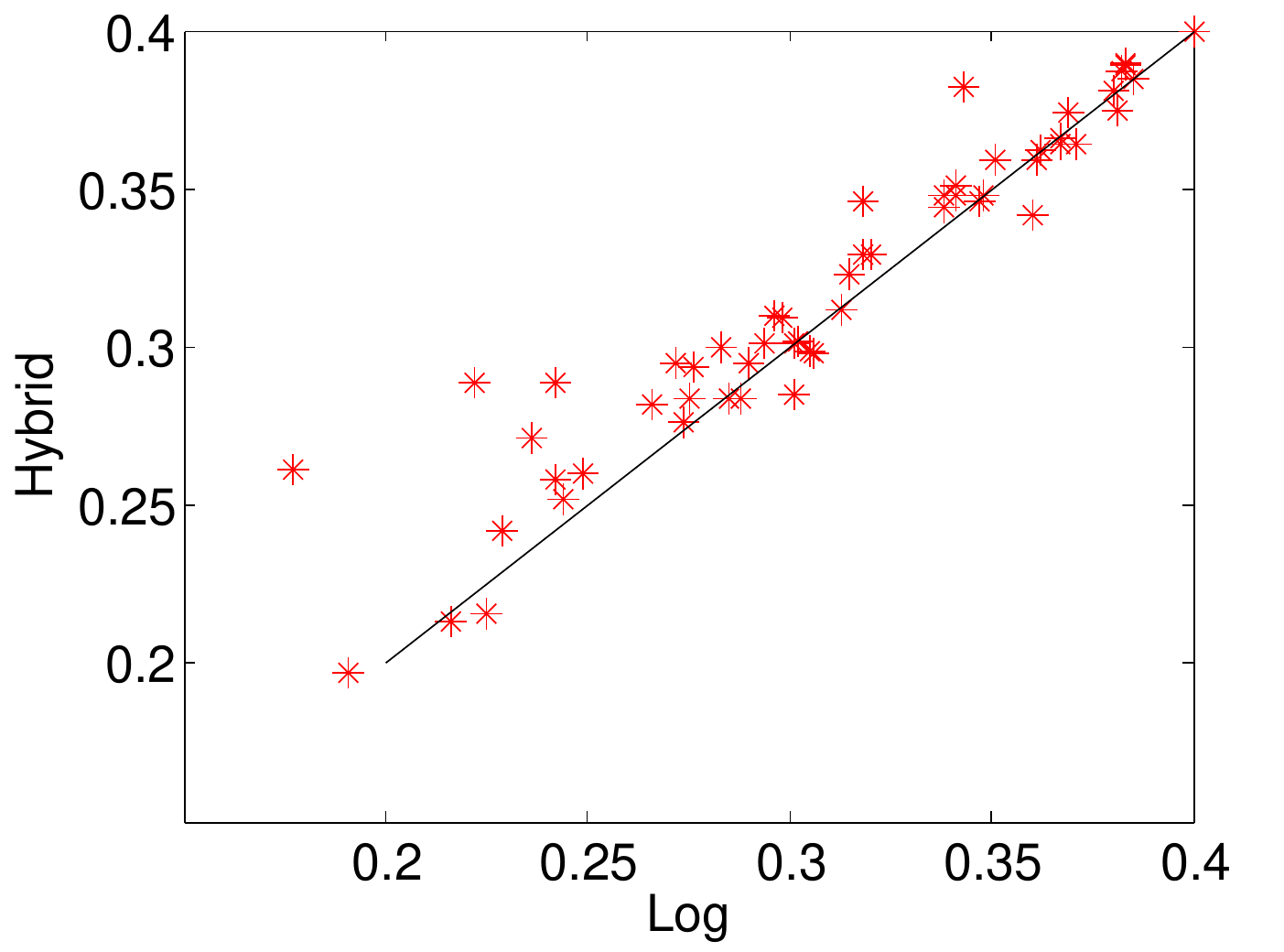}
     \label{fig:hybrid_crf}
     }
      \subfigure[Hinge v.s. Log (30/23)]{
     \includegraphics[width=0.3\linewidth]{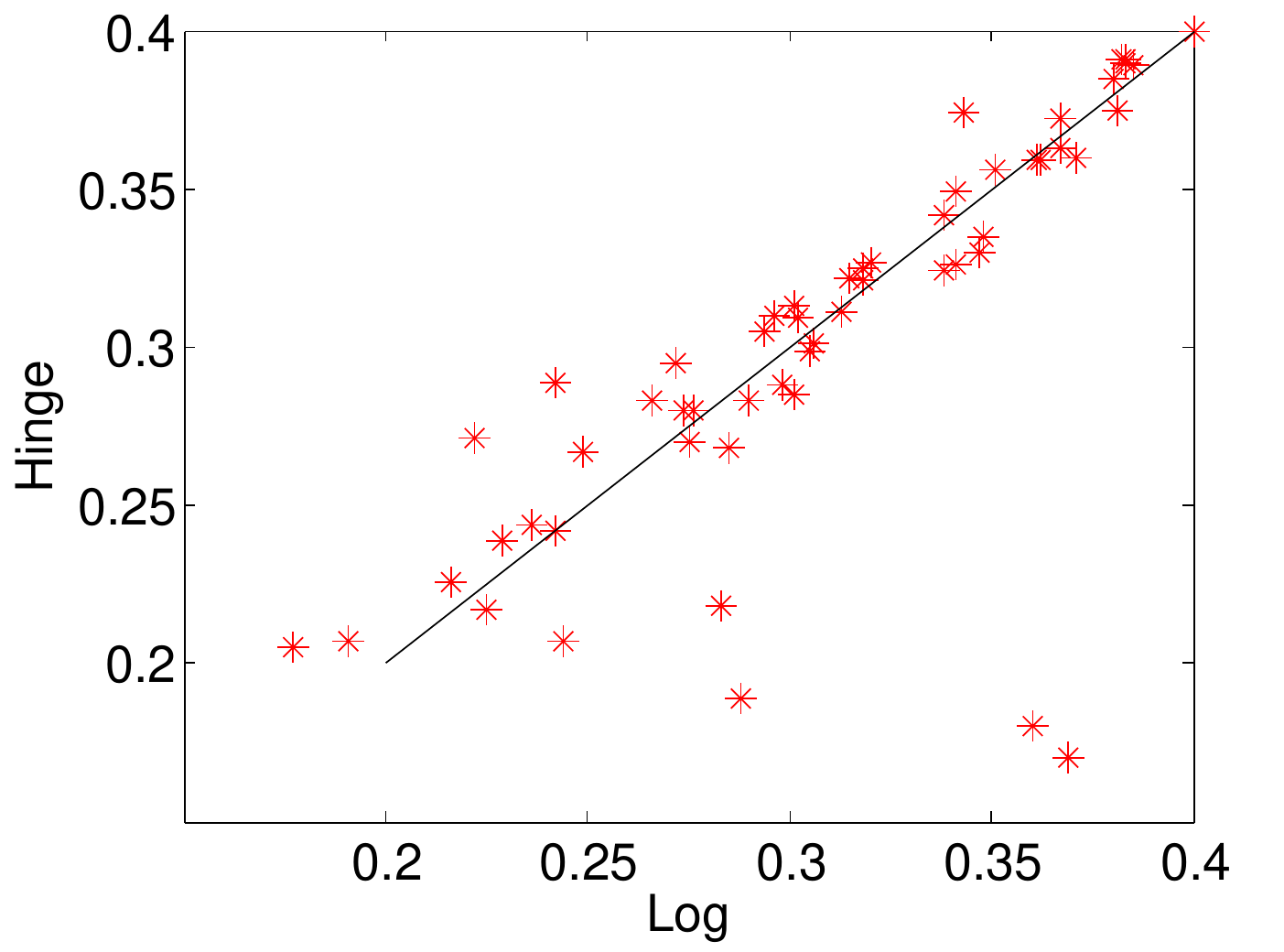}
     \label{fig:svm_crf}
     }
  \caption{Performance of the hybrid, hinge, and log losses on 
  non-dominant/dominant mixtures.
  Points denote pairs of test accuracies for models trained on one of 60 data sets
  using the losses named on the axes.
  Score $(a/b)$ denotes the vertical loss with $a$ wins and $b$ losses (ties
  not counted).}
\label{fig:multiclass_mix}
\end{figure*}

To make the experiment as simple as possible, we considered an observation space of 
size $|\Xcal|=1$ and focused on varying the number of labels and their 
probabilities. 

Fisher consistency analyses the behaviour of losses while observing the entire data population. To mimic seeing the entire data population and the dominant/non-dominant class case, we use a constant vector in $\RR^2$ as features, and learn the parameter vectors $w_y 
\in\RR^2$ for $y\in\Ycal$. The labels $y$ take different values proportionally as follows:
The label set $\Ycal$ took the sizes $|\Ycal|=3,4,5,\dots,10$.
One label $y^*\in\Ycal$ is assigned probability $D_{y^*}(x) = 0.46$ and the 
remainder are given an equal portion of 0.54 (\emph{e.g.}, in the 3 class case the 
other labels each have probability 0.27, and in the 10 class case, 0.06).
Note that this means for all the label set sizes, the gap 
$D_{y^*}(x) - D_{y}(x)$ is at least 0.19 which is always greater than 
$(1-\alpha)(1-2D_{y^*}(x)) = 0.04$ so the hybrid consistency condition 
(\ref{eq:alpha}) is always met. This way, the size of training data does not affect the training error as long as the proportions of $y$ values are not altered. In the same vein, the proportions of $y$ values in the test data are the same as that in the training data, thus test errors and training errors should be the same in this case. Thus we plot the resulting training errors (and the test errors) for hinge, log and hybrid losses in Figure~\ref{fig:multiclass} as a function of the number of labels.
As we can clearly see, the hinge loss error increases as the number
of classes increases, whereas the errors for the log and the hybrid
losses remain a constant ($1-D_{y^*}(x)$), in concordance with the consistency 
analysis.

Models were found using LBFGS \cite{Byretal94} with inexact line search, thus landing on the hinge point almost never happens. In theory, it is a problem --- it may not converge for the non-smooth optimisation problem. But in practice, it works well.

\begin{figure*}[tb]
 \centering
 \subfigure[the testing
  set ]{
     \includegraphics[width=0.4\linewidth]{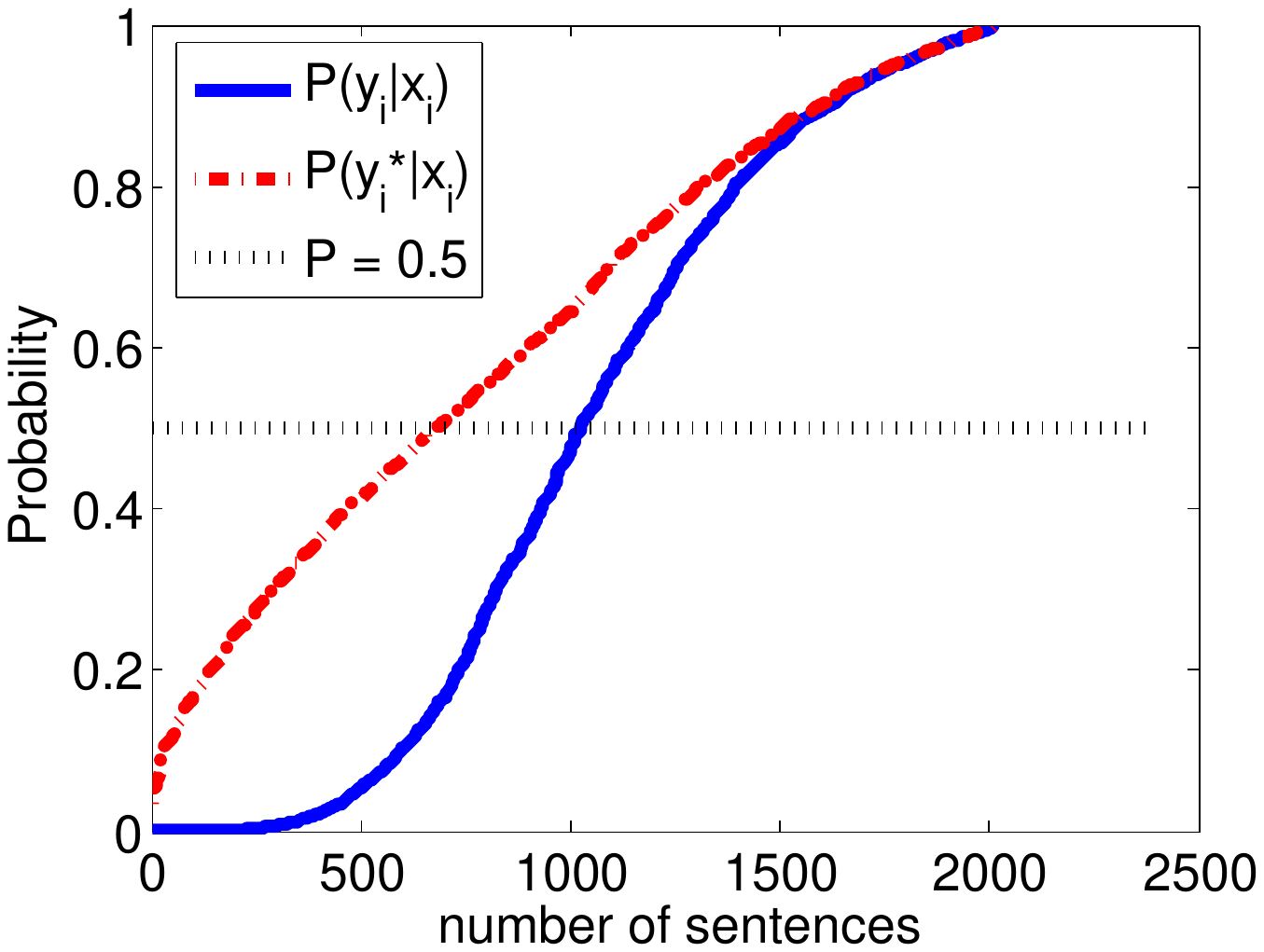}
     \label{fig:chunk_test}
     }
      \subfigure[the training set]{
     \includegraphics[width=0.4\linewidth]{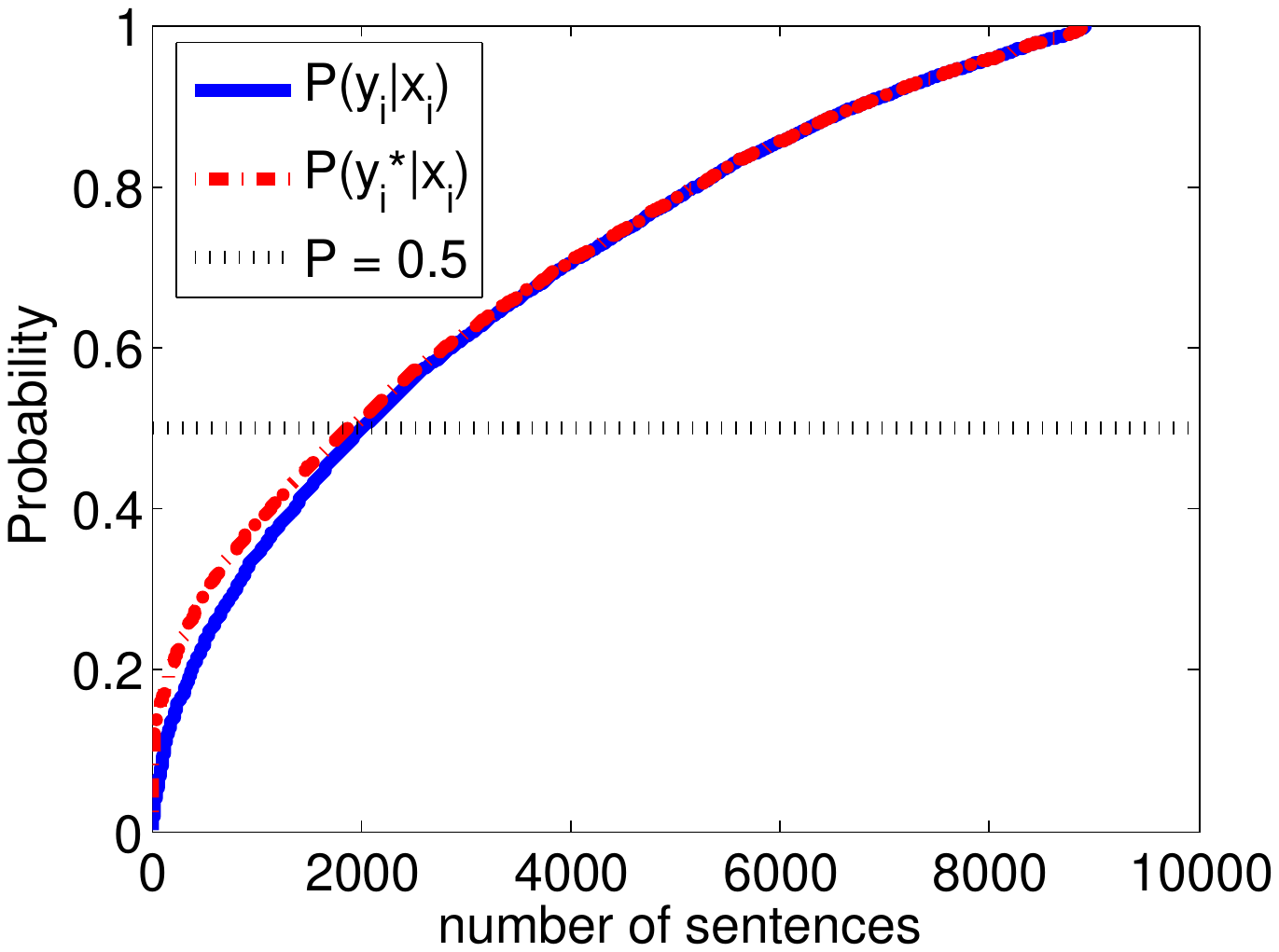}
     \label{fig:chunk_train}
     }
  \caption{Estimated probabilities of the true label $D_{y_i}(x_i)$ and most likely label $D_{y_i^*}(x_i)$. Sentences are sorted according to $D_{y_i}(x_i)$ and $D_{y_i^*}(x_i)$ respectively in ascending
  order. $D=1/2$ is shown as the straight black dot line. About 700
  sentences out of 2012 in the testing set and 2000 sentences out of 8936 in the training set
  have no dominant label.
} \label{fig:chunking}
\end{figure*}

{\Large
\begin{table*}[t]
    \centering {\caption{Accuracy, precision, recall and F1 Score on the
    CONLL2000 text chunking task.} \label{table:conll2000}}
    \begin{tabular}{c|c|c|c|c|c}
    Train Portion &Loss &Accuracy &Precision & Recall &F1 Score\\
    &Hinge&91.14 &85.31 &85.52 &85.41\\
    0.1&Log&92.05 &87.04 &87.01 &87.02\\
    &Hybrid&\bf{92.07} &\bf{87.17} &\bf{86.93} &\bf{87.05}\\
    \hline
    &Hinge&94.61 &91.23 &91.37 &91.30\\
    1&Log&95.10 &92.32 &91.97 &92.15\\
    &Hybrid&\bf{95.11} &\bf{92.35} &\bf{92.00} &\bf{92.17}\\
    \end{tabular}
\end{table*}
}

\subsection{Mix of Non-dominant and Dominant Distributions}

The second synthetic experiment examined how the three losses performed 
given various training set sizes (denoted by $m$) and various proportions of
instances with non-dominant distributions (denoted by $\rho$).

We generated 60 different data sets, all with $\Ycal=\{1,2,3,4,5\}$, 
in the following manner: 
Instances came from either a non-dominant class distribution 
or a dominant class distribution. In the non-dominant class case,
$x\in\RR^{100}$ is set to a predefined, constant, non-zero vector 
and its label distribution is $D_{1}(x) = 0.4$ and $D_y(x) = 0.15$ for 
$y > 1$. 
In the dominant case, each dimension $x_i$ was drawn from a
normal distribution $N(\mu = 1+y,\sigma = 0.6)$ depending on the class 
$y = 1,\dots, 5$. 
The proportion $\rho$ ranged over 10 values $\rho=0.1, 0.2, 0.3,\dots,1$ and for 
each $\rho$, test and validation sets of size 1000 were generated.
Training set sizes of $m = 30, 60, 100, 300, 600, 1000$ were used for each
$\rho$ value for a total of 60 training sets.
The optimal regularisation parameter $\lambda$ and hybrid loss parameter $\alpha$ 
were selected using the validation set for each loss on each training set.
Then models with parameters $w_y\in\RR^{100}$ for $y\in\Ycal$ were found using 
LBFGS \cite{Byretal94} for each of the three losses on each of the 60 training sets and then 
assessed using the test set.

The results are summarised in Figure~\ref{fig:multiclass_mix}. 
Each point shows the test accuracy for a pair of losses. 
The predominance of points above the diagonal lines in a) and b) show that the 
hybrid loss outperforms the hinge loss and the log loss in most of the data sets. 
while the log and hinge losses perform competitively against each other.

\begin{table*}[t]
\centering {\caption{Accuracy, precision, recall and F1 Score on the
baseNP chunking task for training on increasing portions of training
set.} \label{table:basenp}
\begin{tabular}{c|c|c|c|c|c}
Train Portion &Loss &Accuracy &Precision & Recall &F1 Score\\
&Hinge&88.48 &71.70 &75.96 &73.77\\ 
0.1&Log&90.86 &81.09 &78.96 &80.01\\
&Hybrid&\bf{90.90} &\bf{81.23} &\bf{79.09} &\bf{80.15}\\
\hline
&Hinge&94.64 &87.58 &88.30 &87.94\\
1& Log&95.21 &90.07 &88.89 &89.48\\
&Hybrid&\bf{95.24} &\bf{90.12} &\bf{88.98} &\bf{89.55}\\
\end{tabular}
}\\
\end{table*}

\section{Structured Estimation}\label{sec:experiments}
Unlike the general multiclass case, structured estimation problems
have a higher chance of non-dominant distributions because of the
very large number of labels as well as ties or ambiguity regarding
those labels. For example, in text chunking,
changing the tag of one phrase while leaving the rest unchanged
should not drastically change the probability predictions --
especially when there are ambiguities.
Due to the prevalence of non-dominant distributions, we expect
models trained using the hinge loss to perform poorly on these
problems relative to those trained with hybrid or log losses.
We emphasise that our main motivation for investigating structured 
prediction problems is that, as multiclass problems,
they tend to have non-dominant distributions.

%

\subsection{CONLL2000 Text Chunking} Our first structured estimation experiment 
is carried out on the CONLL2000 text chunking task \cite{conll2000}.
The data set has 8936 training sentences and 2012 test sentences
with 106978 and 23852 phrases (a.k.a. chunks), respectively. The
task is to divide a text into syntactically correlated parts of words
such as noun phrases, verb phrases, and so on. For a sentence with
$L$ chunks, its label consists of the tagging sequence of all its
chunks, \emph{i.e.} $y = (y^1,y^2,\dots,y^L)$, where $y^j$ is the chunking
tag for chunk $j$. As is common in this task, the label $y$ is
modelled as a chain-structured graphical model to account for the dependency between
adjacent chunking tags $(y_i^j,y_i^{j+1})$ given observation $x_i$.
Clearly, the model has exponentially many possible labels, which
suggests the absence of a dominant class.

Since the true underlying distribution is unknown, we train a
CRF
on the training set and then apply
the trained model to both testing and training datasets to obtain an
estimate of the conditional distributions for each instance. We sort
the sentences $x_i$ from highest to lowest estimated probability on
the true chunking label $y_i$ given $x_i$. The result is plotted in
Figure~\ref{fig:chunking}, from which we observe the existence of
many non-dominant distributions --- about 1/3 of the testing
sentences and about 1/4 of the training sentences.

We use the feature template from the CRF++ toolkit \cite{crfplusplus}, and
    the CRF code from Leon Bottou \cite{crfsgd}. Stochastic Gradient Descent (SGD) \cite{crfsgd} is used for training. During training, dynamic programming (\ie Viterbi algorithm) for inference is used.
We split the data into 3 parts: training ($20\%$), testing ($40\%$)
and validation ($40\%$).  
The regularisation parameter $\lambda$ and
the weight $\alpha$ were determined via parameter selection using the
validation set. To see the performance with different training
sizes, we took part of the training data to learn the model and gathered statistics
on the test set. The accuracy, precision, recall
and F1 Score on the test set are reported in Table~\ref{table:conll2000}
when using 10\% and 100\% of the training set. 
The hybrid loss marginally outperforms both the hinge loss and the log loss.

\subsection{baseNP Chunking}
A similar methodology to the previous experiment is applied to the BaseNP data set  
\cite{crfplusplus}. It has 900 sentences in
total and the task is to automatically classify a chunking phrase as
baseNP or not. 

Again SGD \cite{crfsgd} is used for training. During training, dynamic programming (\ie Viterbi algorithm) for inference is used.
We split the data into 3 parts: training ($20\%$),
testing ($40\%$) and validation ($40\%$). Once again, $\lambda$ and $\alpha$ are
determined via model selection on the validation set. We report the
test accuracy, precision, recall and F1 Score in
Table~\ref{table:basenp} for training on increasing proportions of
the training set. The hybrid marginally outperforms the other two losses on all
measures.

\subsection{Human action recognition}
Here we consider recognising human actions in TV episodes, where each contains one or more persons and may interact with each other. We evaluate our method on the TVHI dataset \citep{patron:2010high}, which contains 300 short videos collected from TV episodes and includes five action classes: \emph{handshake} (HS), \emph{hug} (HG), \emph{high-five} (HF), \emph{kiss} (KS) and \emph{others} (OT). Here a person labelled the action \emph{others} means that there is no interaction between this person and any other persons in the image. Each video contains a number (up to seven) of people performing one of the five action classes. The ground-truth provided with the dataset includes upper body bounding boxes, discrete body poses, the action labels and the interaction status between any pair of persons (\ie a binary variable indicating whether there is an interaction). We manually choose 2,188 images from this dataset and divide these examples into three sets without intersection: the training set (400 frames), the validating set (894 frames) and the testing set (894 frames). Here $\alpha$ is determined via model selection on the validation set. Note in \cite{patron:2010high} their task is to predict both interactions and actions, whereas here our task is to predict actions given interaction status. More specifically, our goal is to solve the estimation problem of finding the actions $y\in \Ycal$ of all subjects in an observation image $x \in \Xcal$, given pairwise interaction status.

We use PGMs to model the dependency of the actions in the same image. Consider a graph $G=(\Vcal,\Ecal)$ with each node $i \in \Vcal$ representing an action variable $y^i$ and each edge $(i,j) \in \Ecal$ reflecting the dependency of the two action variables. The edge set $\Ecal$ is constructed according to the annotated interaction status. If there is an interaction between two persons in the annotation, then an edge between two corresponding nodes is added to the edge set $\Ecal$. 

We cast this estimation problem as finding an energy function $E(x,y)$ such that for an observation image $x \in \Xcal$, we assign the actions that receive the smallest energy with respect to $E$, that is
\begin{equation}
y ^*= \argmin_{y\in \Ycal} E(x,y;w).
\label{eq:infer}
\end{equation}
Here we use the energy function $E$ with unary terms $U$ and pairwise terms $S$ as follows,
\begin{equation}
E(x,y;w)=\sum_{i\in \Vcal}U(y^i,x;w')+\sum_{(i,j)\in \Ecal}S(y^i,y^j,x;w'').
\end{equation}
where $w=[w';w'']$, and
\begin{subequations}
\begin{align}
&U(y^i,x;w')=\inner{\phi_1(x^i,y^i)}{w'},\\
&S(y^i,y^j,x;w'')=\inner{\phi_2(x^i,x^j,y^i,y^j)}{w''}.
\end{align}
\label{eq:energy}
\end{subequations}
Here $\phi_1$ and $\phi_2$ are node and edge features (which we will define later),  and $x^i$ is the sub-image of the bounding box on the $i$-th subject. The model parameter $w$ will be learned during training. 

Our feature representation is a combination of several visual cues including multiclass SVM action classification scores, human body poses and the relative position between two individuals, which have been exploited to distinguish different actions in \cite{patron:2010high, patron:2012structured}. Here we combine these visual cues in a similar way as \cite{patron:2010high}. To be specific, let $\eb_1(y^i)\in \{0,1\}^5$ represent an unit vector with the $y^i$-th dimension equals 1 (0 elsewhere). Similarly, let $\eb_2(r^{ij})\in \{0,1\}^6$ denote another unit vector with the $r^{ij}$-th dimension equals 1 (0 elsewhere). Here $r^{ij}$ denotes the relative position of person $j$ to $i$. To compute $r^{ij}$, we employ a simple method in \cite{patron:2010high} which only requires the bounding boxes of person $i$ and $j$. Each $r^{ij}$ value represents a relative position in the set $\{overlap, adjacent-left, adjacent-right, near-left, near-right, far\}$. Let $\otimes$ denote the Kronecker product. The features are defined as
\begin{subequations}
\begin{align}
&\phi_1(x^i,y^i)=\db^i\otimes \eb_1(y^i),\label{eq:node_feat}\\
&\phi_2(x^i,x^j,y^i,y^j)=\eb_1(y^i)\otimes\eb_1(y^j)\otimes\eb_2(r^{ij})\otimes\ob^i\otimes\ob^j\notag\\
&\qquad\qquad+\eb_1(y^j)\otimes\eb_1(y^i)\otimes\eb_2(r^{ij})\otimes\ob^j\otimes\ob^i,\label{eq:edge_feat}
\end{align}
\label{eq:feat}
\end{subequations}
where $\db^i\in \RR^5$ is a score vector contains the action classification scores obtained by applying a multiclass SVM classifier to the histograms of gradients (HoG) descriptor \cite{dalal2005histograms} extracted from the bounding box area of person $i$. Similarly, $\ob^i\in \RR^5$ represents another score vector of the pose classification scores (the descriptors for pose classification is the same as that used for action classification). Here we consider five body pose classes $\{profile-left, profile-right, frontal-left, frontal-right, backwards\}$. To extract the HoG features, we superimpose an $8\times 8$ grid on the bounding box area and accumulate HoG for each grid cell using five orientation bins. The final descriptor is a concatenation of the sub-descriptors of all cells.

Intuitively, the node term $U$ reflects the confidence of assigning person $i$ the action label $y^i$ observing $x^i$. The edge term $S$ encodes the correlation between actions $y^i$ and $y^j$ observing $x^i,x^j$.

Implementing Kronecker product naively to compute energies could be memory and time consuming. Fortunately, we see that the feature vectors in \eq{eq:feat} are highly sparse (especially the edge feature). Thus we only need to multiply the non-zero components of the feature vectors with corresponding components of the $w$ without using Kronecker product. 

\begin{figure*}
\centering
\subfigure[hinge loss]{\includegraphics[width=.32\textwidth]{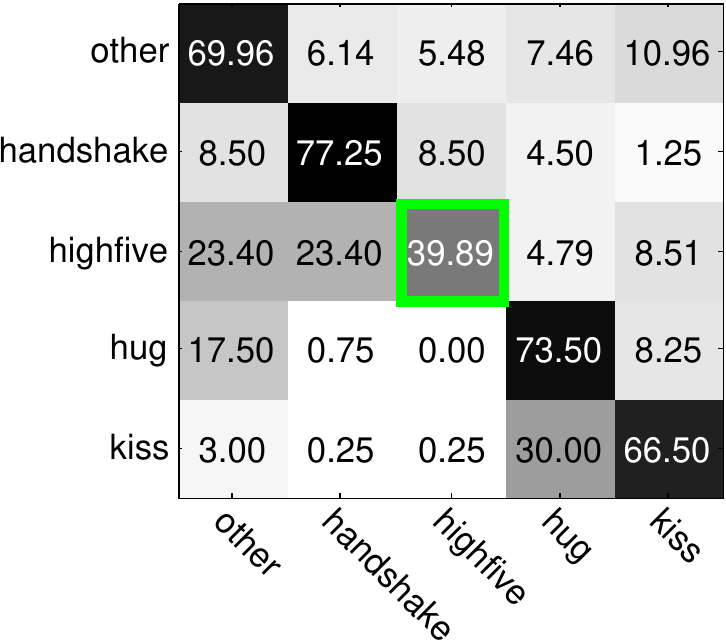}}
\subfigure[log loss]{\includegraphics[width=.32\textwidth]{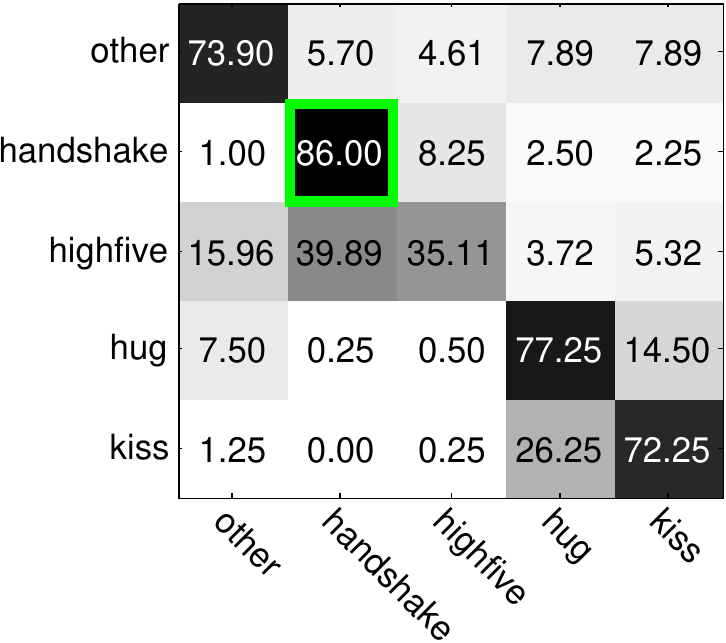}}
\subfigure[hybrid loss]{\includegraphics[width=.32\textwidth]{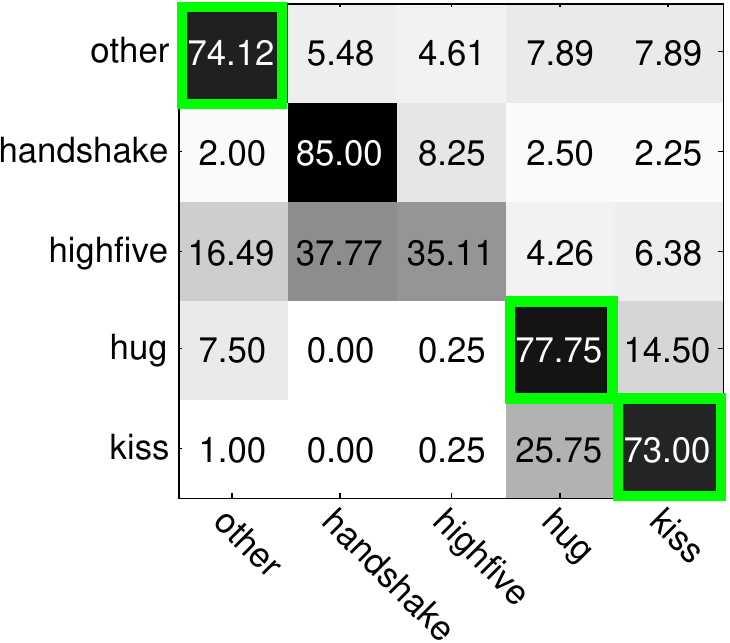}}
\caption{Confusion matrices on the TVHI dataset. For each class the best are highlighted by green rectangles. The hybrid loss achieves the best classification accuracy on three out of five action classes, \ie \emph{other, hug} and \emph{kiss}.}
\label{fig:conf_mat_tvhi}
\end{figure*}

\begin{figure*}
\centering
\subfigure{\includegraphics[width=44mm]{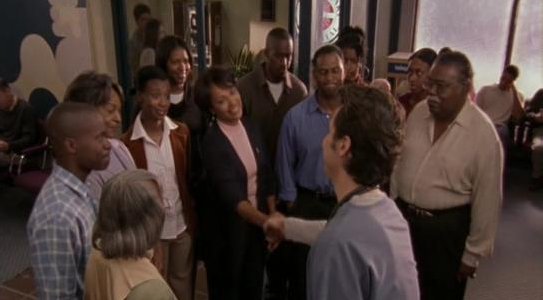}}
\subfigure{\includegraphics[width=44mm]{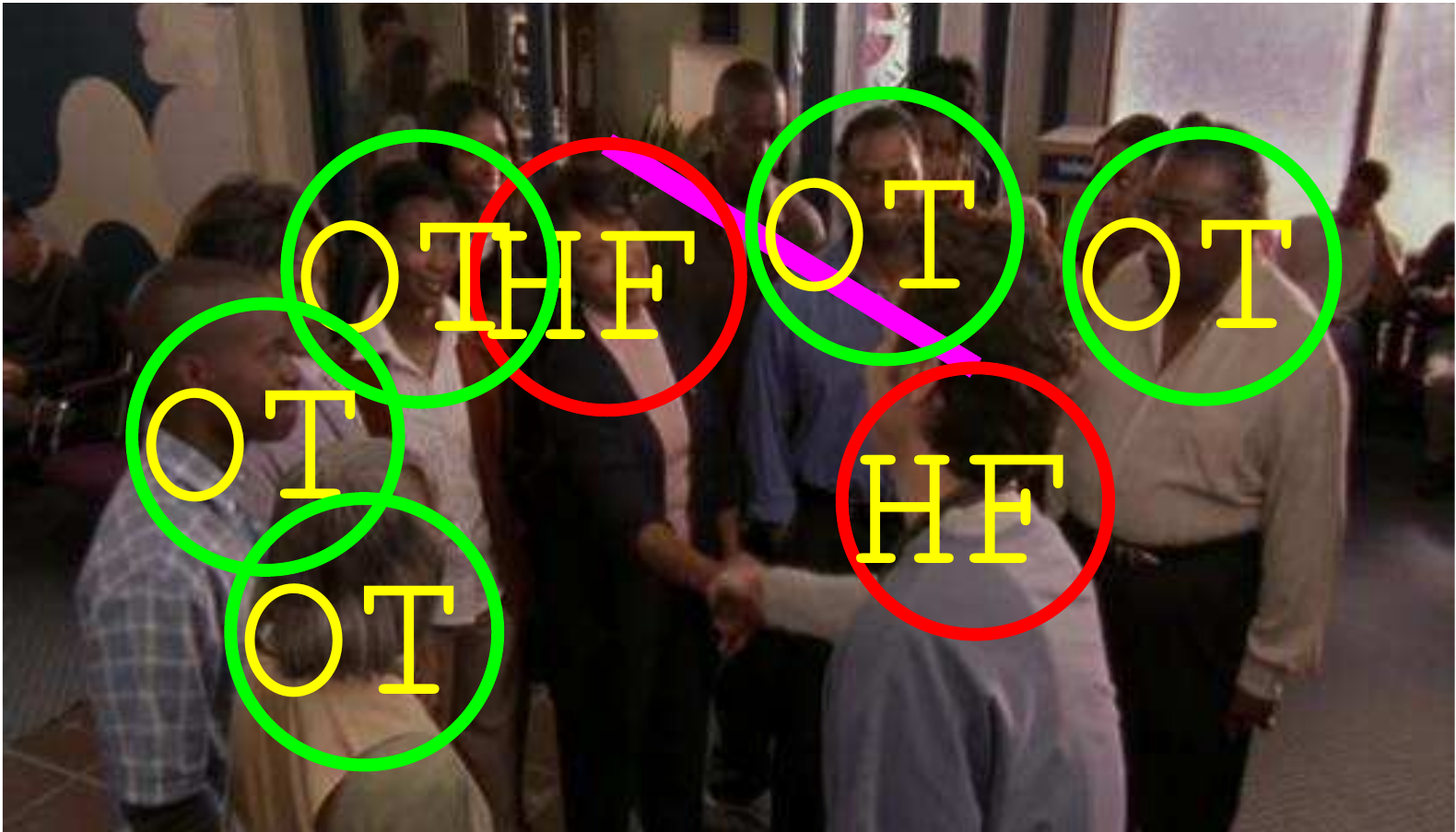}}
\subfigure{\includegraphics[width=44mm]{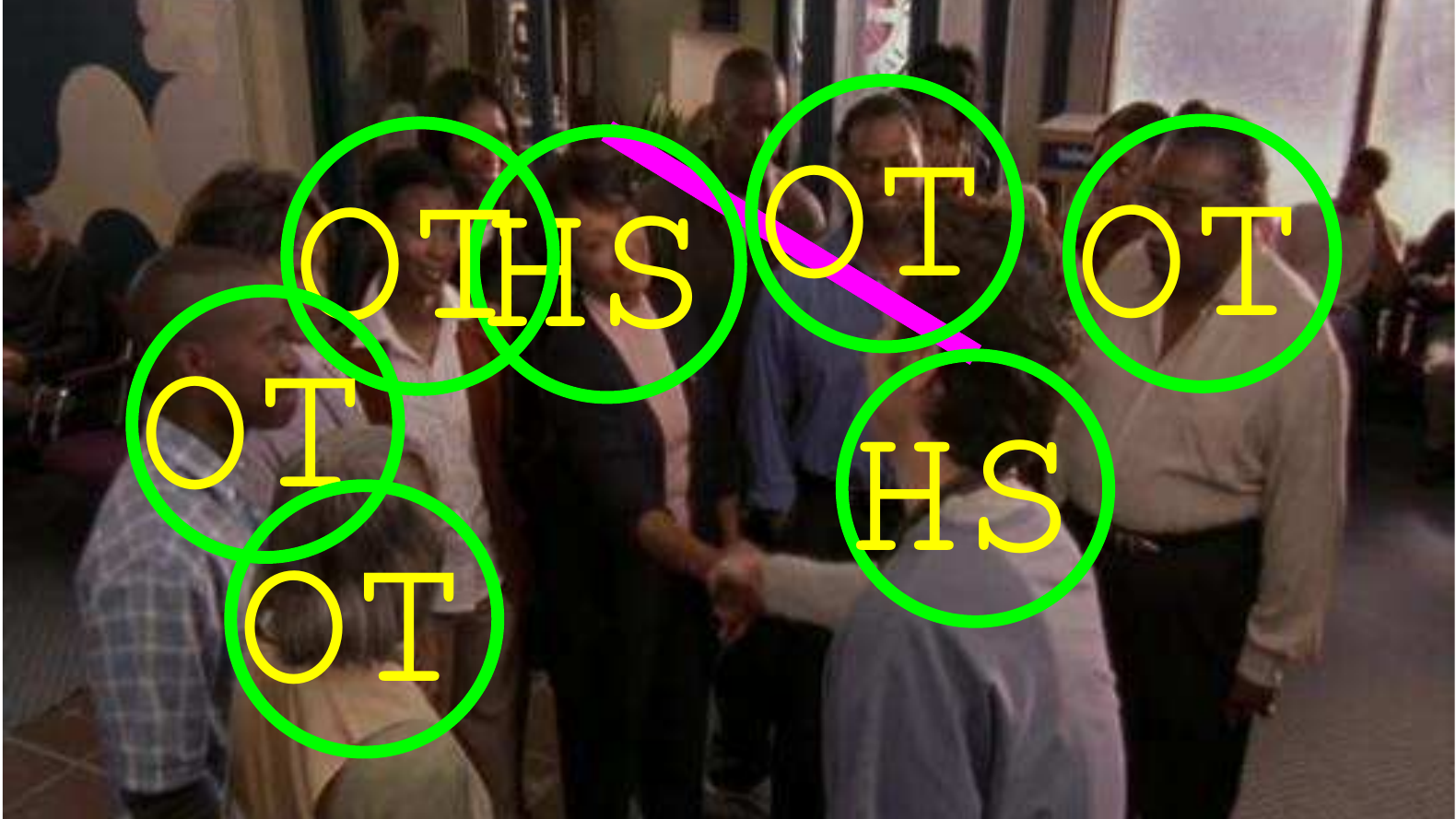}}
\subfigure{\includegraphics[width=44mm]{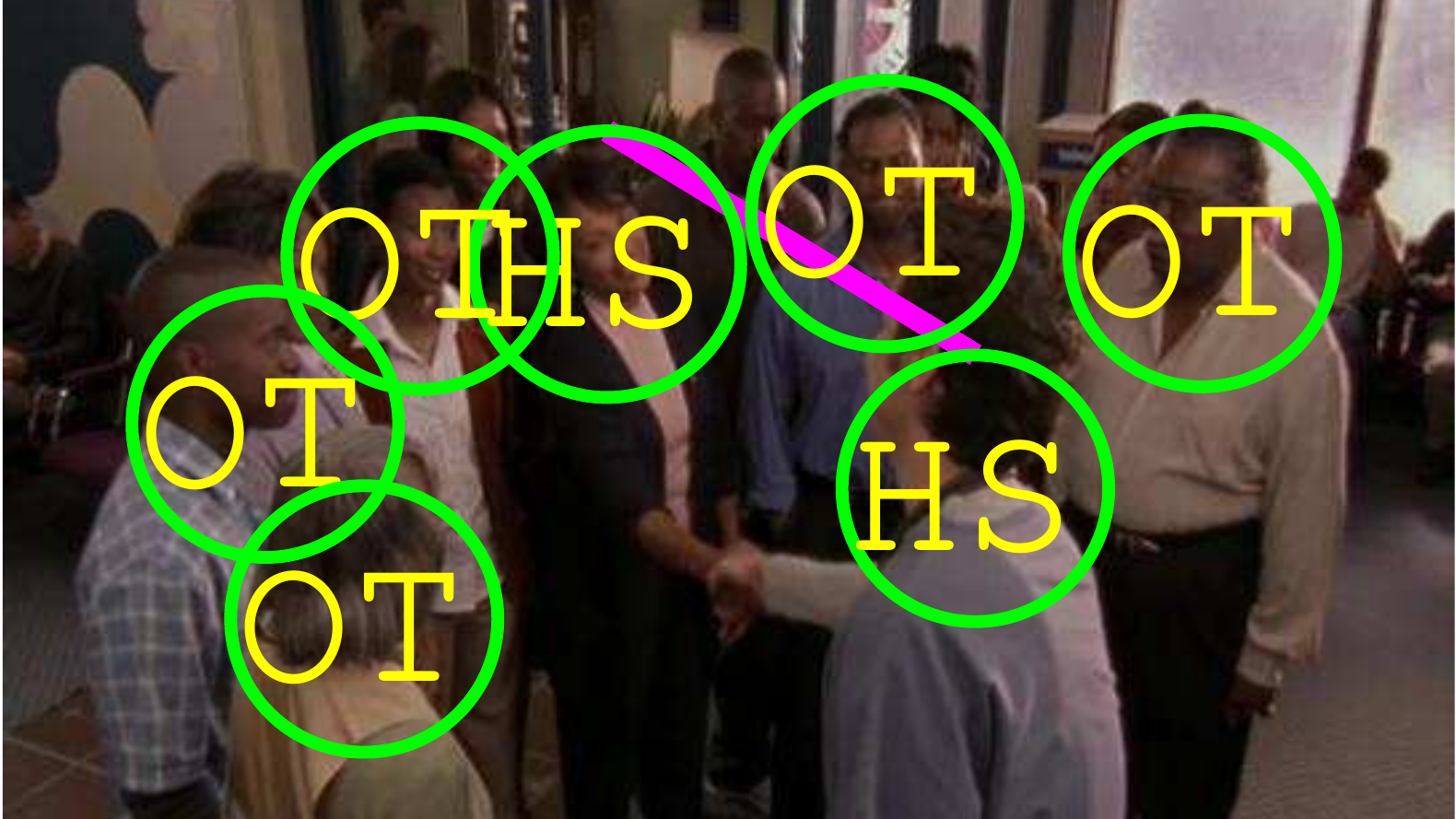}}
\subfigure{\includegraphics[width=44mm]{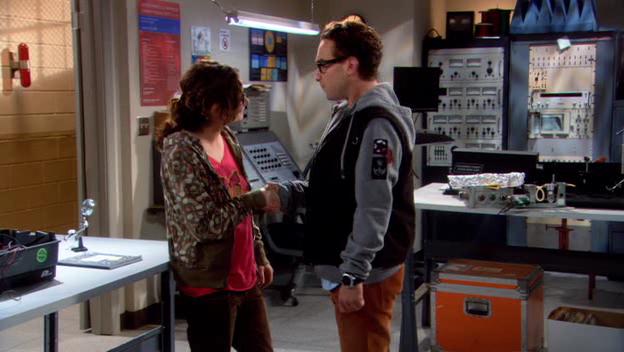}}
\subfigure{\includegraphics[width=44mm]{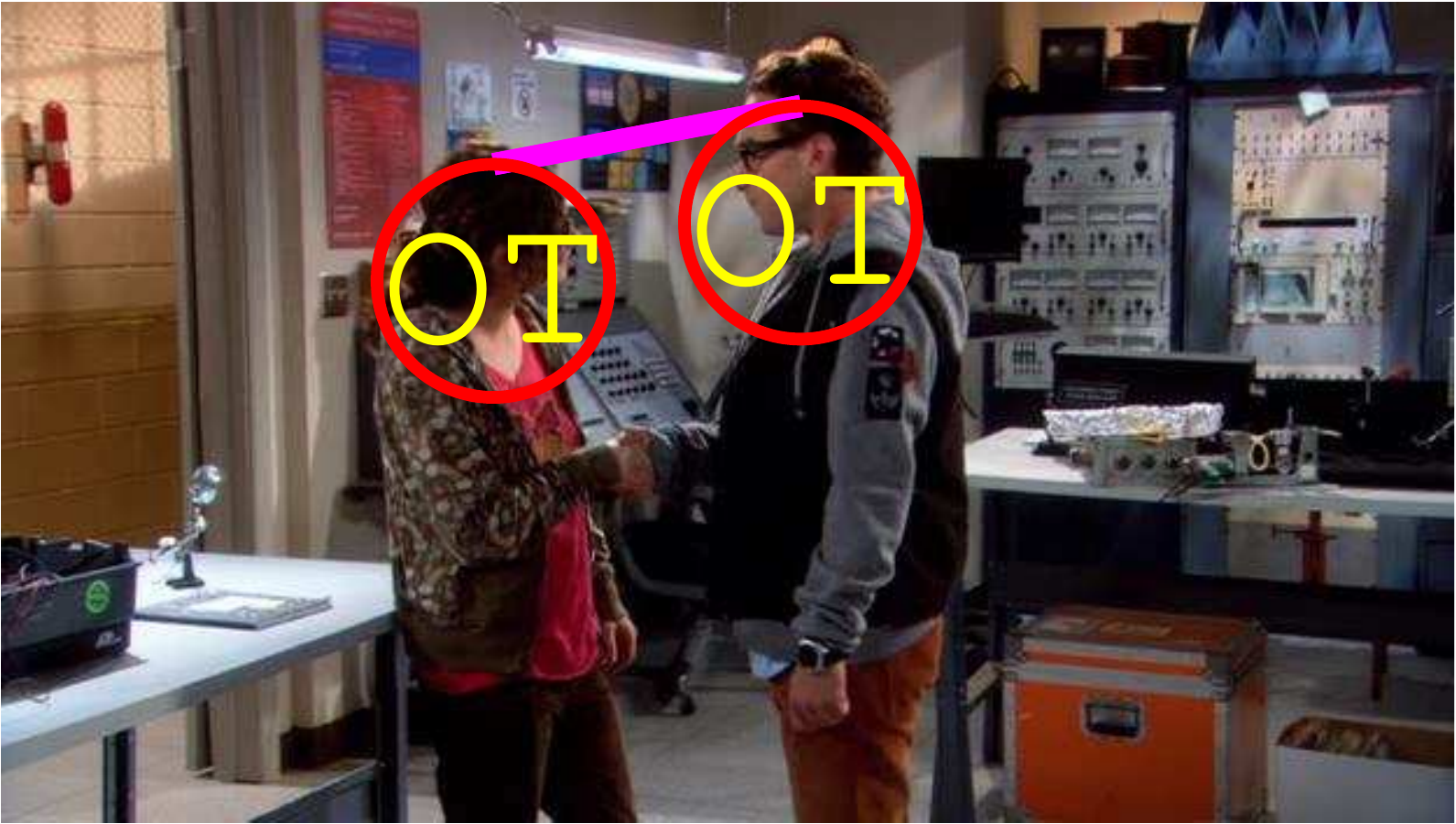}}
\subfigure{\includegraphics[width=44mm]{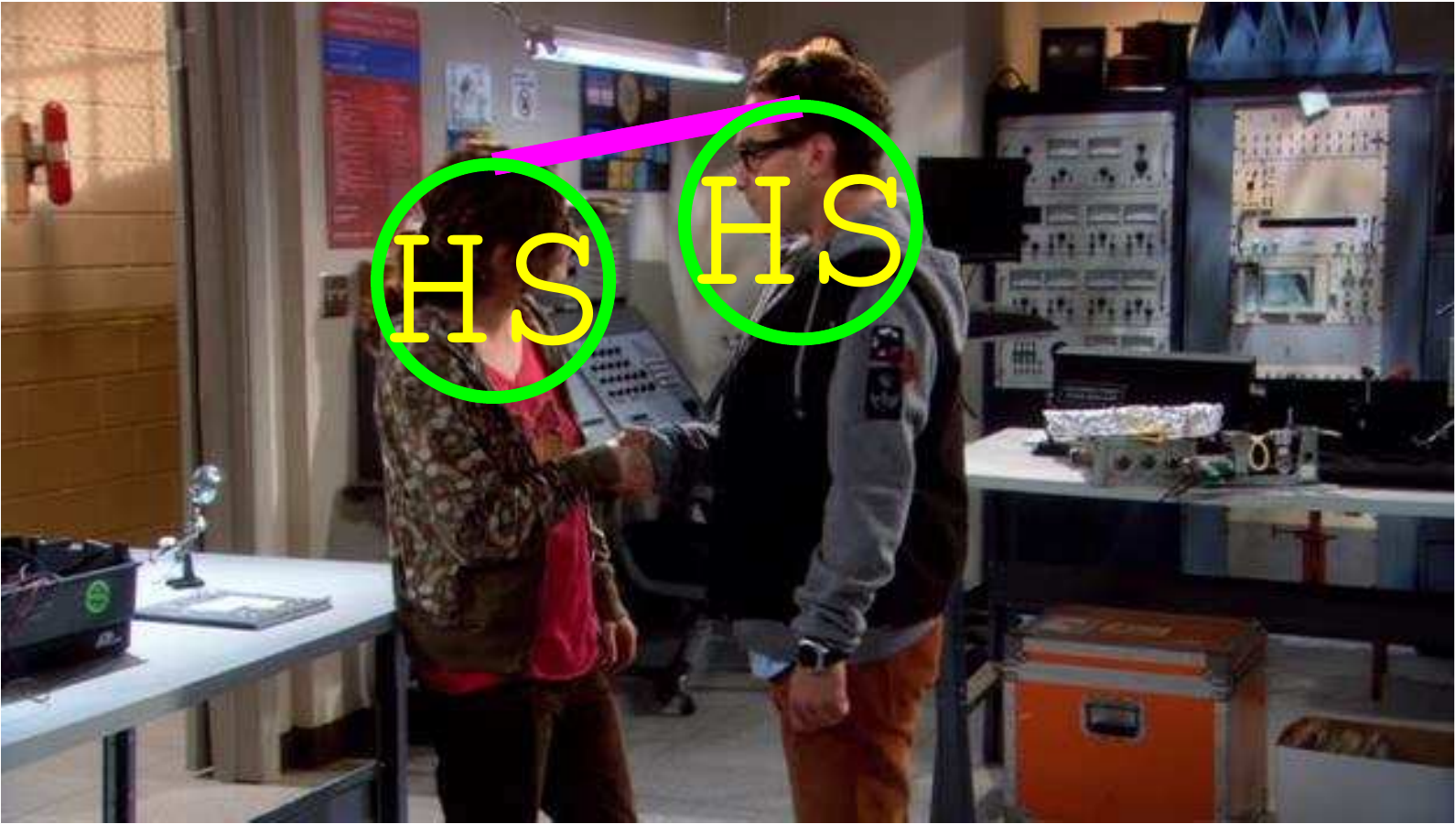}}
\subfigure{\includegraphics[width=44mm]{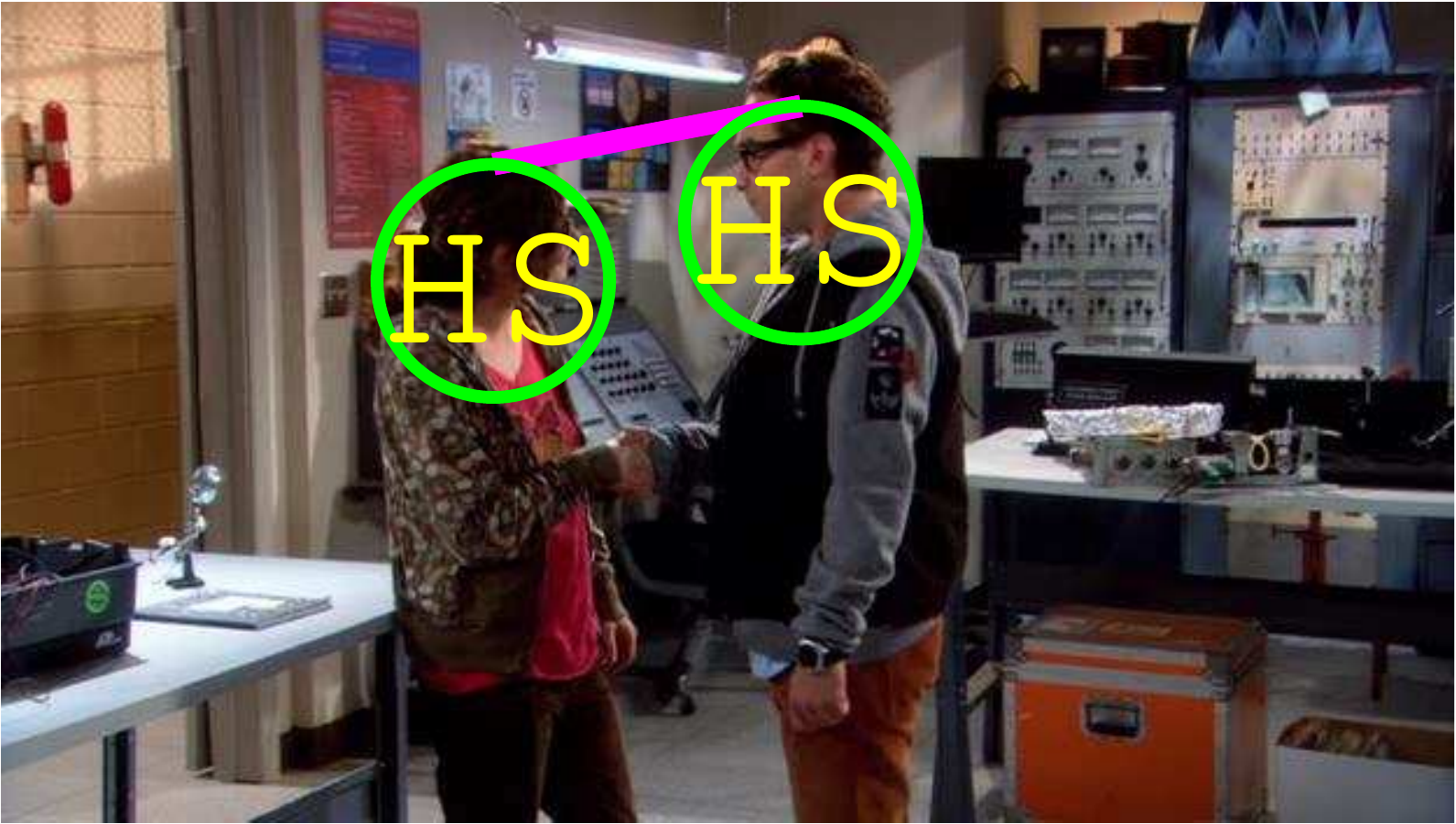}}
\subfigure{\includegraphics[width=44mm]{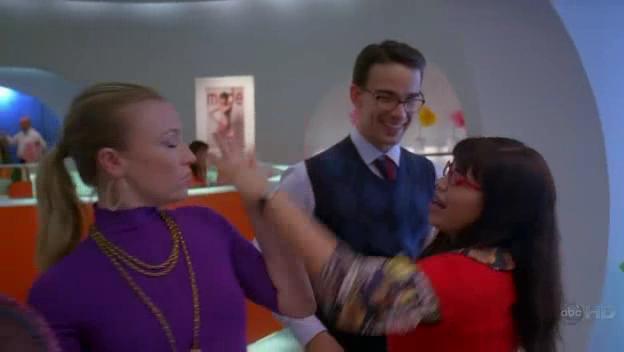}}
\subfigure{\includegraphics[width=44mm]{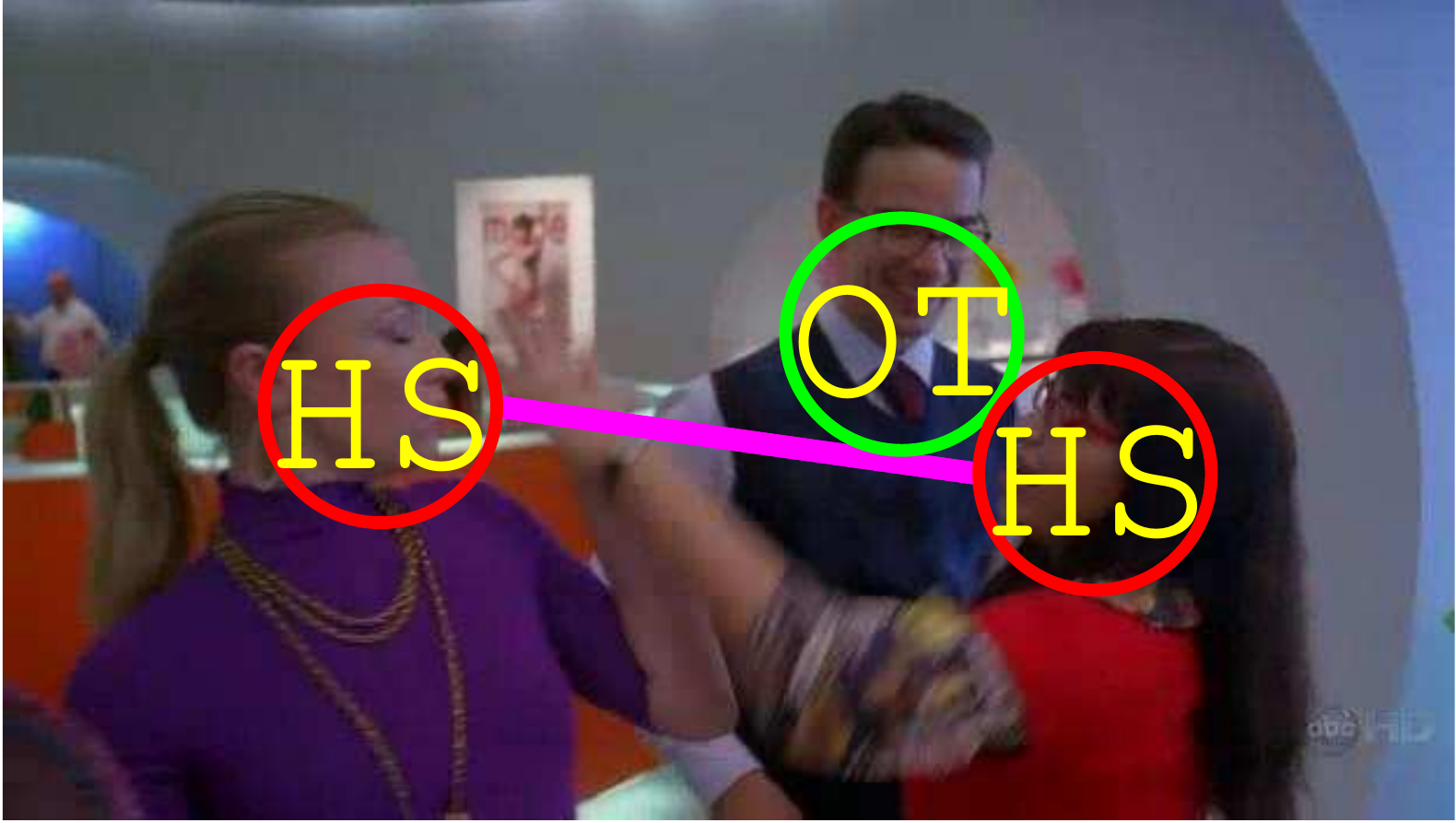}}
\subfigure{\includegraphics[width=44mm]{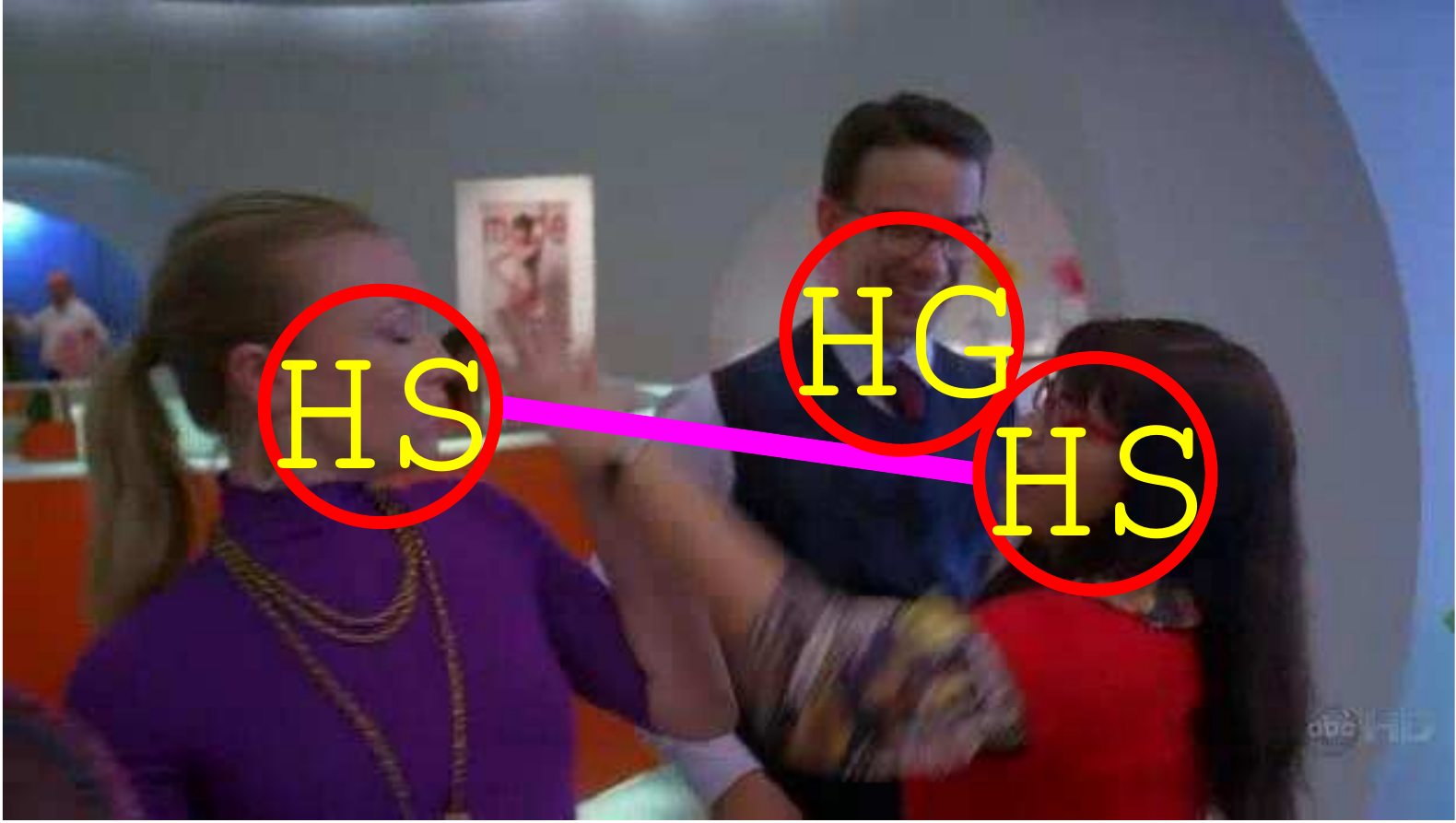}}
\subfigure{\includegraphics[width=44mm]{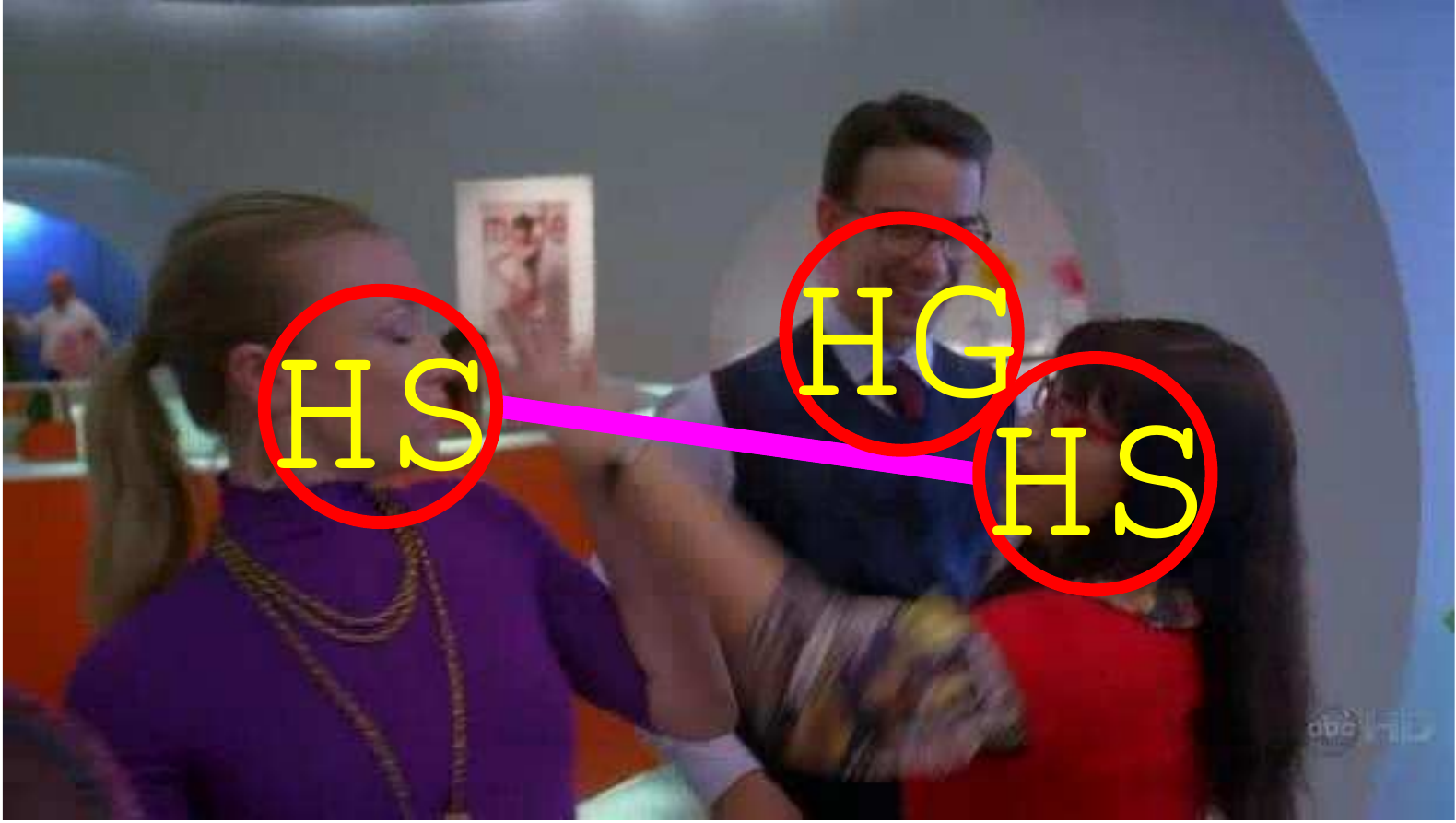}}
\subfigure{\includegraphics[width=44mm]{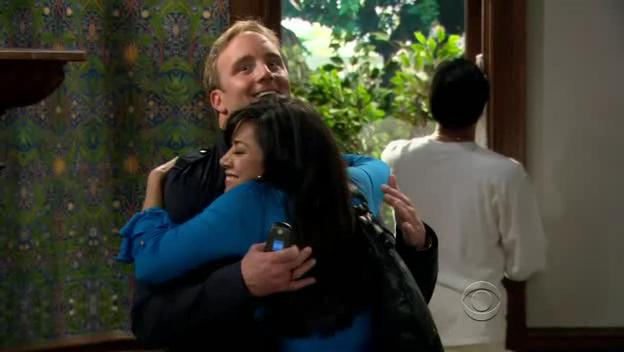}}
\subfigure{\includegraphics[width=44mm]{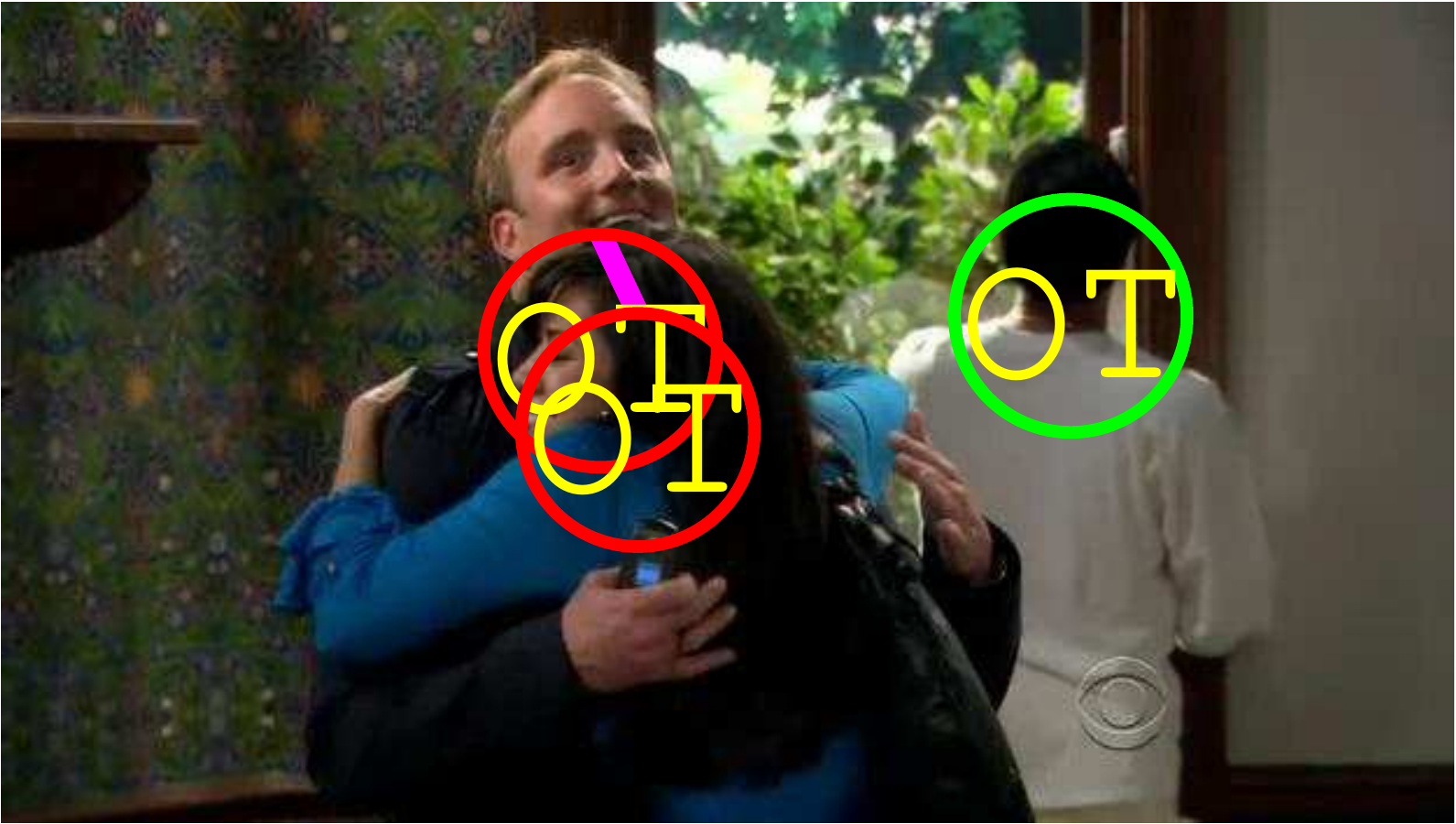}}
\subfigure{\includegraphics[width=44mm]{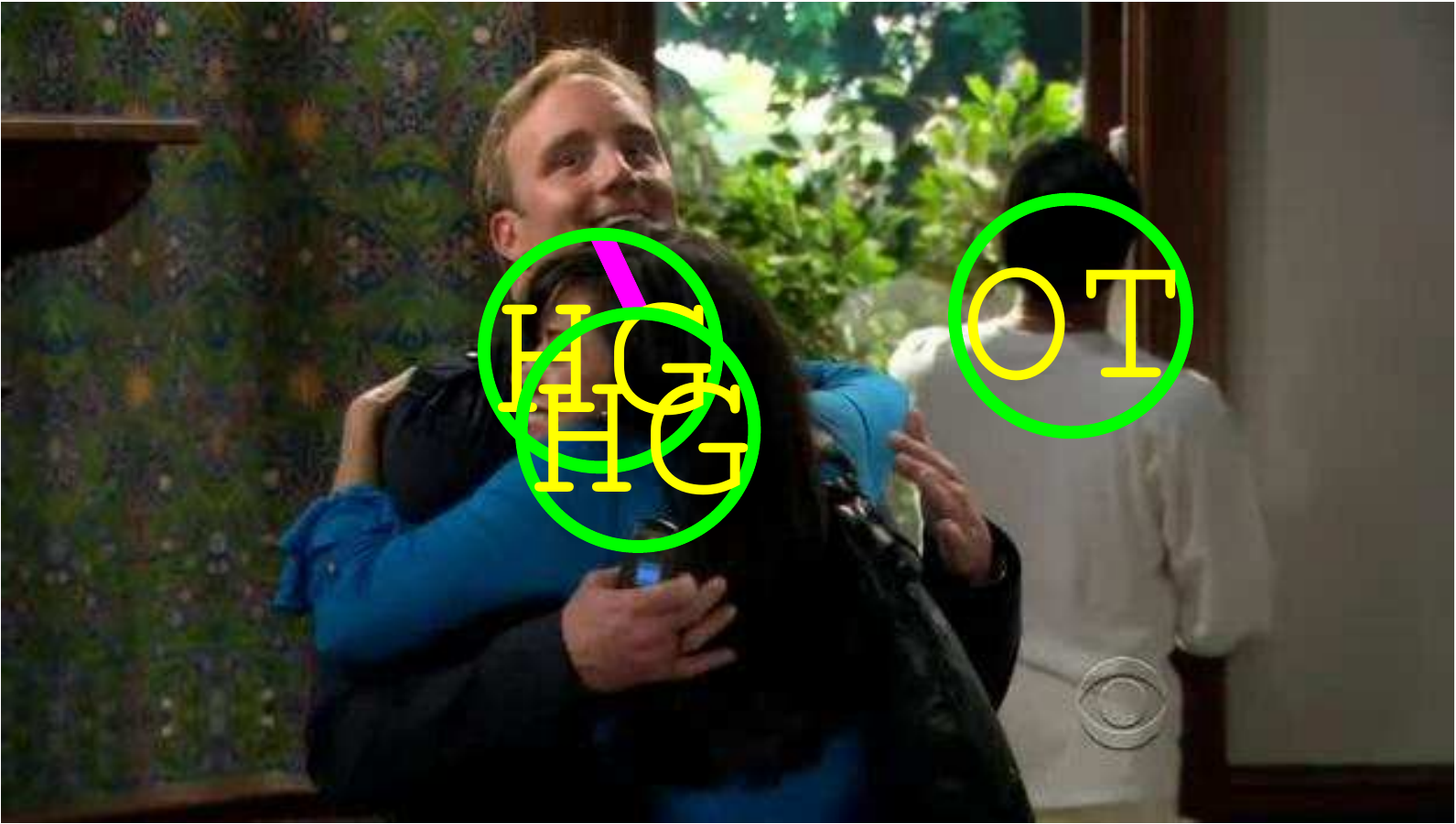}}
\subfigure{\includegraphics[width=44mm]{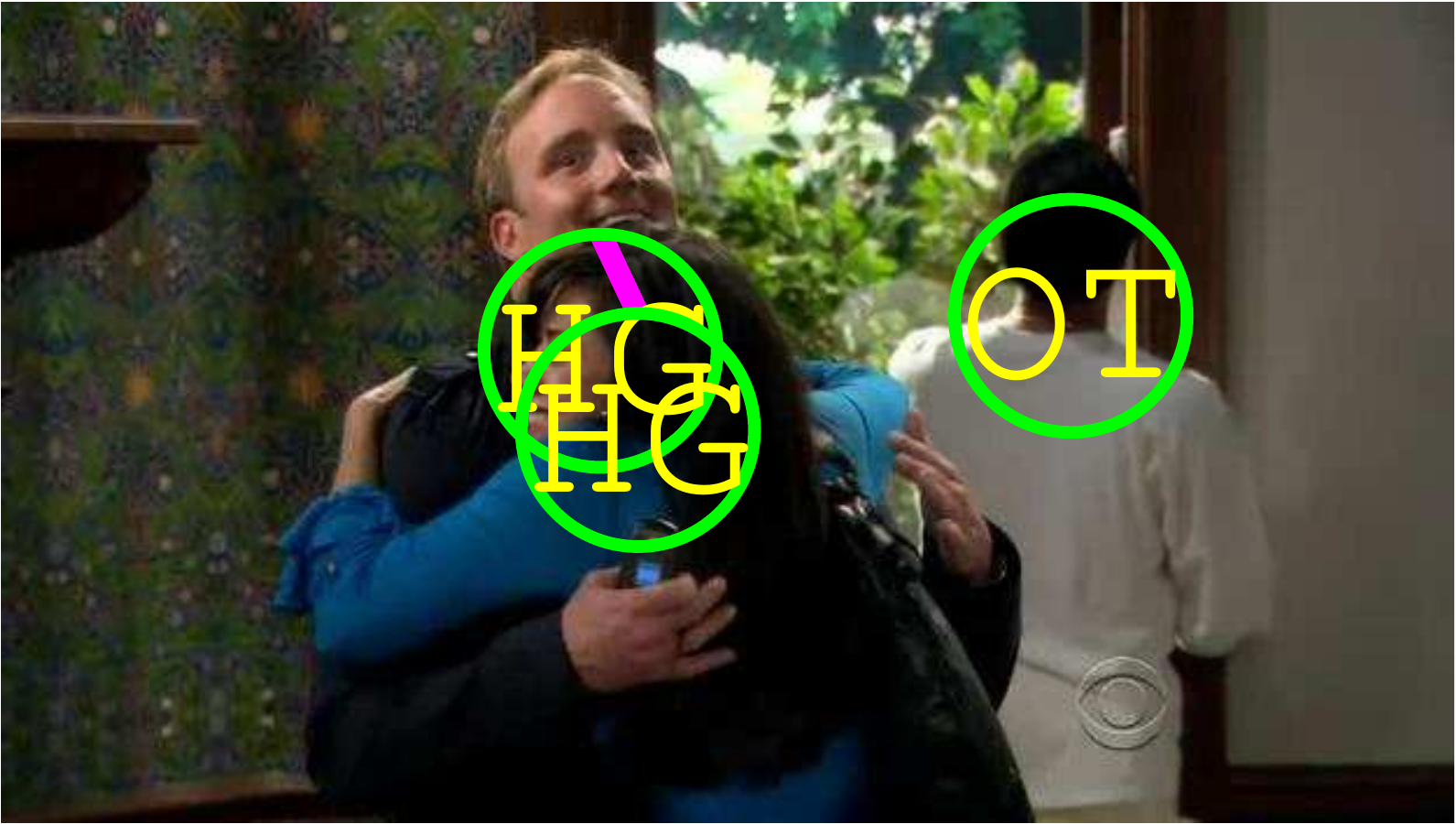}}
\subfigure{\includegraphics[width=44mm]{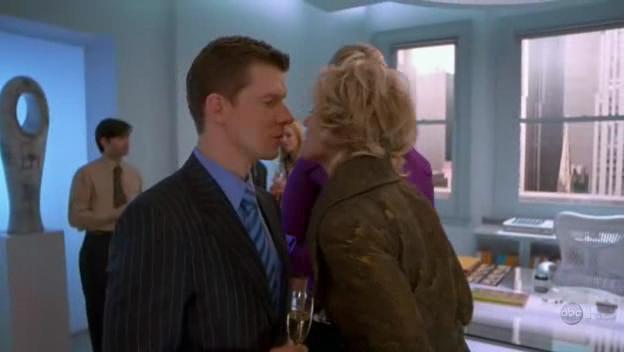}}
\subfigure{\includegraphics[width=44mm]{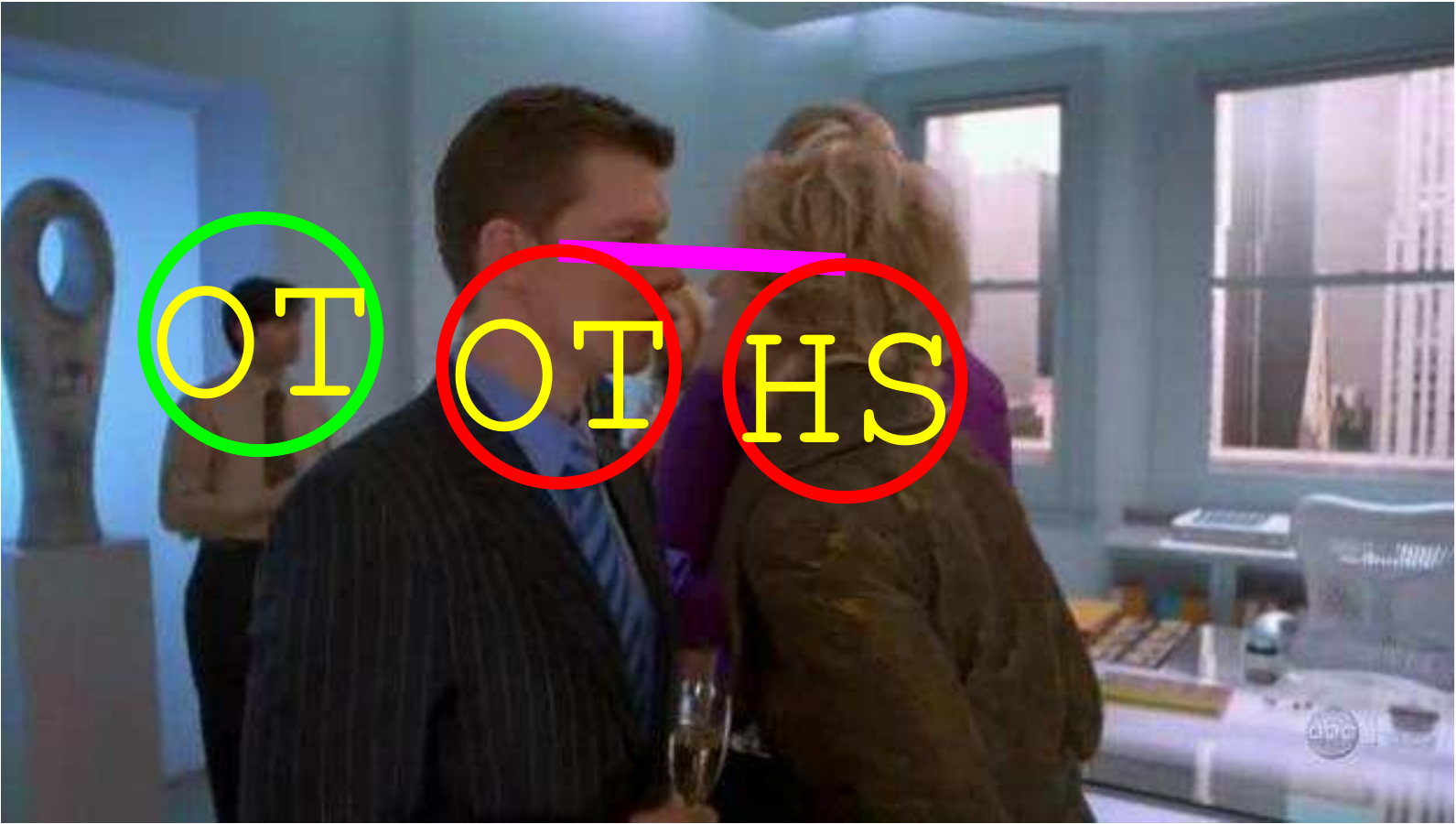}}
\subfigure{\includegraphics[width=44mm]{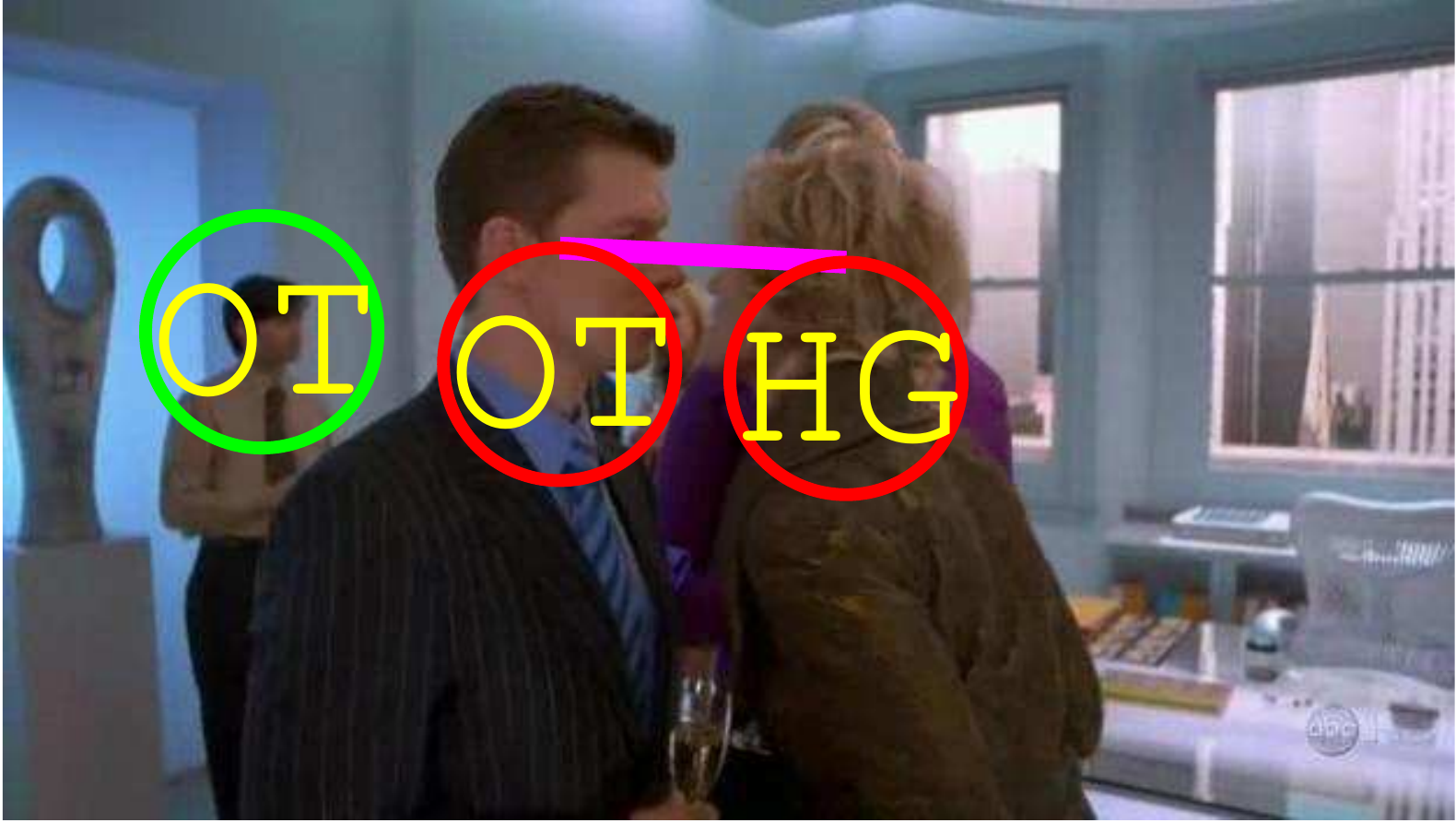}}
\subfigure{\includegraphics[width=44mm]{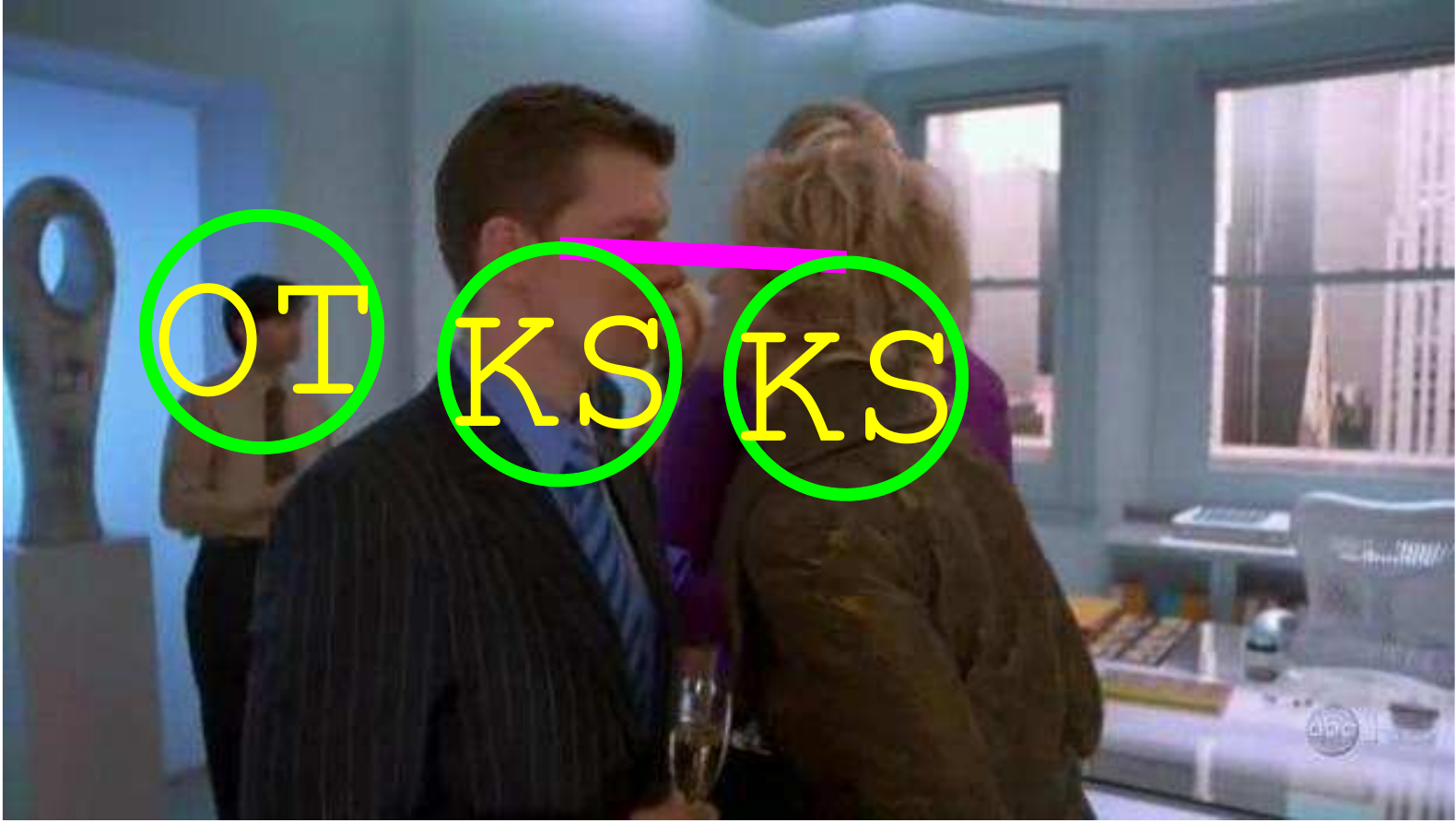}}
\caption{Visualisation of some action predictions using different losses. {\bf FIRST COLUMN}:  the input images; {\bf SECOND COLUMN}: the hinge loss results; {\bf THIRD COLUMN}: the log loss results; {\bf FOURTH COLUMN}: the hybrid loss results. The PGMs are superimposed on the images with green node and red node indicate correct and incorrect predictions respectively. Note OT, HS, HF, HG, KS denote five action classes: \emph{others, handshake, highfive, hug, kiss}. The edges (line segments in pink) come from the interaction annotation in the dataset, and are used to model the dependency between two action variables of subjects.}
\label{fig:tvhi_pred_vis}
\end{figure*}


Following the work of SVMs for structured prediction \cite{TsoJoaHofAlt05}, the hinge loss in this case is
\begin{align}
\ell_{H}(x,y,w) =[\Delta(y^{\dag},y) + E(x,y;w) -  E(x,y^{\dag};w)]_+.
\label{eq:ghinge}
\end{align}
where 
\begin{align}
y^{\dag}_{i} & =  \argmin_{y}(E(y,x_i;w)-\Delta_{\text{}}(y,y_i)).
\label{eq:sub-sup}
\end{align}
Here $\Delta$ is a label cost function (or distance) that describes the discrepancy of two labels.
We use the popular Hamming distance, which is defined as
\begin{align}
\label{eq:hamming_loss}
\Delta_{\text{Ham}}(y,y') = \frac{1}{L}\sum_{j=1}^L \delta(y^{j}\ne y'^{j}),
\end{align}
where $y = (y^{j})_{j=1}^{L}$ and 
$y' =(y'^{j})_{j=1}^{L}$.
%




For the log loss, let $E(x,y;w) = -f_{y}(x)$. Thus, according to \eq{eq:prob} we have
\begin{align}
p(y|x;w) = \frac{\exp(-E(x,y;w))}{\sum_{y' \in \Ycal} \exp(-E(x,y';w))}.
\end{align}
So the log loss is
\begin{align}
&\ell_{L}(x,y,w) =  -\log\Big(\frac{\exp(-E(x,y;w))}{\sum_{y' \in \Ycal} \exp(-E(x,y';w))}\Big)\\
& = E(x,y;w) + \log\Big({\sum_{y' \in \Ycal} \exp(-E(x,y';w))}\Big).
\label{eq:log-loss-z}
\end{align}

According to \eq{eq:hybrid} the hybrid loss is 
\begin{align}
\ell_{\alpha}(x,y,w) = \alpha \ell_L(x,y,w)\!+\!(1-\alpha)\ell_H(x,y,w).
\label{eq:hybrid-har}
\end{align}
%
The sub-gradient of the hybrid loss is simply a convex combination of the sub-gradient of the hinge loss and the gradient of the log loss. It is known that the sub-gradient of the hinge loss can be computed via standard MAP inference techniques, and the gradient of the log loss can be computed via standard marginal inference techniques. We use the max-product algorithm for the hinge loss and the sum-product algorithm for the log loss.   

To accelerate the training, we apply the stochastic sub-gradient method from \citep{shalev2007pegasos} to the hybrid loss. Here the maximum number of iterations is set to be 30 and the min-batch size is set to be 10. Once the parameters are learned, we use the standard max-product algorithm to make prediction on testing data.



In order to evaluate the recognition performance of different losses, we show the confusion matrices in Figure~\ref{fig:conf_mat_tvhi}. It can be seen that the hybrid loss achieves the best true positive rates on 3 classes (OT, HG and KS) out of 5 action classes, while the log loss and the hinge loss perform best on the HS class and the HF class respectively. Note all losses perform much worse on the HF class than the rest classes. This is because the training set is highly biased as the number of persons performing the \emph{high-five} action in the training set is much less than other classes.

We also give some recognition examples as that shown in Figure~\ref{fig:tvhi_pred_vis}. The first column shows four input images, each containing multiple persons with occlusions making the recognition task difficult. As we can see, hinge loss performs worst with 10 persons out of 18 mislabelled. The log loss outperforms the hinge loss in general as 5 persons are misclassified. For the hybrid loss, persons in all images are perfectly classified except for the third image, where all 3 persons are misclassified.

\section{Conclusion and Discussion}\label{sec:conclusions}

We have provided theoretical and empirical motivation for the use of
a novel hybrid loss for multiclass and structured prediction
problems which can be used in place of the more common log loss or
multiclass hinge loss. This new loss attempts to blend the strength
of purely discriminative approaches to classification, such as
Support Vector machines, with probabilistic approaches, such as
Conditional Random Fields. Theoretically, the hybrid loss enjoys
better consistency guarantees than the hinge loss while
experimentally we have seen that the addition of a purely
discriminative component can improve accuracy when data is less
prevalent.

In general the consistency condition may not hold if $\alpha$ is selected by cross-validation. For example, when the selected $\alpha$ is very small. However, we observe the selected $\alpha$ values in our experiments are always very close to 1.

\subsection{Future Work}
Theoretically, we expect that some stronger sufficient conditions on $\alpha$ are
possible since the bounds used to establish Theorem~\ref{thm:fcc}
are not tight. Our conjecture is that a necessary and sufficient
condition would include a dependency on the number of classes.
We are also investigating connections between $\alpha$ and the multiclass Tsybakov 
noise condition \citep{Che06}.

To our knowledge, the notion of a regular function class for the purposes of
consistency analysis is novel.
Characterisations of this property for existing parametric models would make 
testing for regularity easier.

In structured prediction, there is still a big gap between the analysis and the practice. For example, in structured prediction, we know the parametric hinge loss is not consistent for binary label cost function, but we don't know whether the parametric hybrid loss is. Moreover, we don't have theoretical results for general label cost functions. To better connect our theory with actual practice on structured prediction problems, we plan to investigate consistency for general cost functions (e.g. Hamming loss) that are more commonly used in these problems.



\ifCLASSOPTIONcompsoc
  \section*{Acknowledgments}
\else
  \section*{Acknowledgment}
\fi

The bulk of this research was performed while Q. Shi was with NICTA. NICTA is funded by the Australian Government as represented by the Department of Broadband, Communications and the Digital Economy and the Australian Research Council through the ICT Centre of Excellence program.

This research was partly supported under Australian Research Council Discovery Projects funding scheme (DP0988439 and DP140102270) and Australian Research Council Discovery Early Career Researcher Award funding scheme (DE120101161) and supported by The Australian Centre for Visual Technologies and The Computer Vision group of The University of Adelaide.

We would like to thank the anonymous reviewers of an earlier version of this work 
for their constructive feedback, particularly regarding the Theorems 1 and 2.

\ifCLASSOPTIONcaptionsoff
  \newpage
\fi



\bibliographystyle{plain}
\bibliography{bibfile,smmcrf}

\begin{thebibliography}{10}

\bibitem{BakHofSchSmoetal07}
G.~Bakir, T.~Hofmann, B.~Sch{\"o}lkopf, A.~Smola, B.~Taskar, and S.~V.~N.
  Vishwanathan.
\newblock {\em Predicting Structured Data}.
\newblock MIT Press, Cambridge, Massachusetts, 2007.

\bibitem{crfsgd}
Leon Bottou.
\newblock Stochastic gradient descent for conditional random fields (crfs),
  2010.
\newblock v1.3 \url{http://leon.bottou.org/projects/sgd}.

\bibitem{Byretal94}
Richard~H. Byrd, Jorge Nocedal, and Robert~B. Schnabel.
\newblock Representations of quasi-newton matrices and their use in limited
  memory methods.
\newblock {\em Mathematical Programming}, 1994.

\bibitem{Che06}
Di-Rong Chen and Tao Sun.
\newblock Consistency of multiclass empirical risk minimization based on convex
  loss.
\newblock {\em JMLR}, 7:2435--2447, 2006.

\bibitem{conll2000}
CoNLL.
\newblock Shared task for conference on computational natural language learning
  (conll-2000), 2000.
\newblock \url{http://www.cnts.ua.ac.be/conll2000/chunking/}.

\bibitem{CraSin00}
K.~Crammer and Y.~Singer.
\newblock On the learnability and design of output codes for multiclass
  problems.
\newblock In N.~Cesa-Bianchi and S.~Goldman, editors, {\em Proc.\ Annual Conf.\
  Computational Learning Theory}, pages 35--46, San Francisco, CA, 2000. Morgan
  Kaufmann Publishers.

\bibitem{dalal2005histograms}
N.~Dalal and B.~Triggs.
\newblock Histograms of oriented gradients for human detection.
\newblock In {\em CVPR}, 2005.

\bibitem{koller2009probabilistic}
D.~Koller and N.~Friedman.
\newblock {\em Probabilistic graphical models: principles and techniques}.
\newblock MIT press, 2009.

\bibitem{crfplusplus}
Taku Kudo.
\newblock Crf++: Yet another crf toolkit, 2010.
\newblock v0.53 \url{http://crfpp.sourceforge.net/}.

\bibitem{LafMcCPer01}
J.~D. Lafferty, A.~McCallum, and F.~Pereira.
\newblock Conditional random fields: Probabilistic modeling for segmenting and
  labeling sequence data.
\newblock In {\em Proc.\ Intl.\ Conf.\ Machine Learning}, volume~18, pages
  282--289, San Francisco, CA, 2001. Morgan Kaufmann.

\bibitem{Liu07}
Yufeng Liu.
\newblock Fisher consistency of multicategory support vector machines.
\newblock In {\em Proc.\ Intl.\ Conf.\ Machine Learning}, 2007.

\bibitem{LugVay04}
G.~Lugosi and N.~Vayatis.
\newblock On the bayes-risk consistency of regularized boosting methods.
\newblock {\em The Annals of Statistics}, 32(1):30--55, 2004.

\bibitem{patron:2012structured}
A.~Patron-Perez, M.~Marszalek, I.~Reid, and A.~Zisserman.
\newblock Structured learning of human interactions in tv shows.
\newblock {\em TPAMI}, 34(12):2441--2453, 2012.

\bibitem{patron:2010high}
A.~Patron-Perez, M.~Marszalek, A.~Zisserman, and I.~Reid.
\newblock High five: Recognising human interactions in tv shows.
\newblock In {\em British Machine Vision Conference}, 2010.

\bibitem{Reid:2009b}
M.D. Reid and R.C. Williamson.
\newblock Composite binary losses.
\newblock {\em Journal of Machine Learning Research}, 11, September 2010.

\bibitem{ShaPer03}
F.~Sha and F.~Pereira.
\newblock Shallow parsing with conditional random fields.
\newblock In {\em Proceedings of HLT-NAACL}, pages 213--220, Edmonton, Canada,
  2003. Association for Computational Linguistics.

\bibitem{shalev2007pegasos}
Shai Shalev-Shwartz, Yoram Singer, and Nathan Srebro.
\newblock Pegasos: Primal estimated sub-gradient solver for svm.
\newblock In {\em International conference on Machine learning}, 2007.

\bibitem{TewBar07}
A.~Tewari and P.L. Bartlett.
\newblock On the consistency of multiclass classification methods.
\newblock {\em Journal of Machine Learning Research}, 8:1007--1025, 2007.

\bibitem{TsoJoaHofAlt05}
I.~Tsochantaridis, T.~Joachims, T.~Hofmann, and Y.~Altun.
\newblock Large margin methods for structured and interdependent output
  variables.
\newblock {\em J. Mach. Learn. Res.}, 6:1453--1484, 2005.

\bibitem{Zhang:2009}
Z.~Zhang, M.~I. Jordan, W.~J. Li, and D.~Y. Yeung.
\newblock Coherence functions for multicategory margin-based classification
  methods.
\newblock In {\em Proceedings of the Twelfth Conference on Artificial
  Intelligence and Statistics (AISTATS)}, 2009.

\end{thebibliography}

%

\begin{IEEEbiography}[{\includegraphics[width=1in,height=1.25in,clip,keepaspectratio]{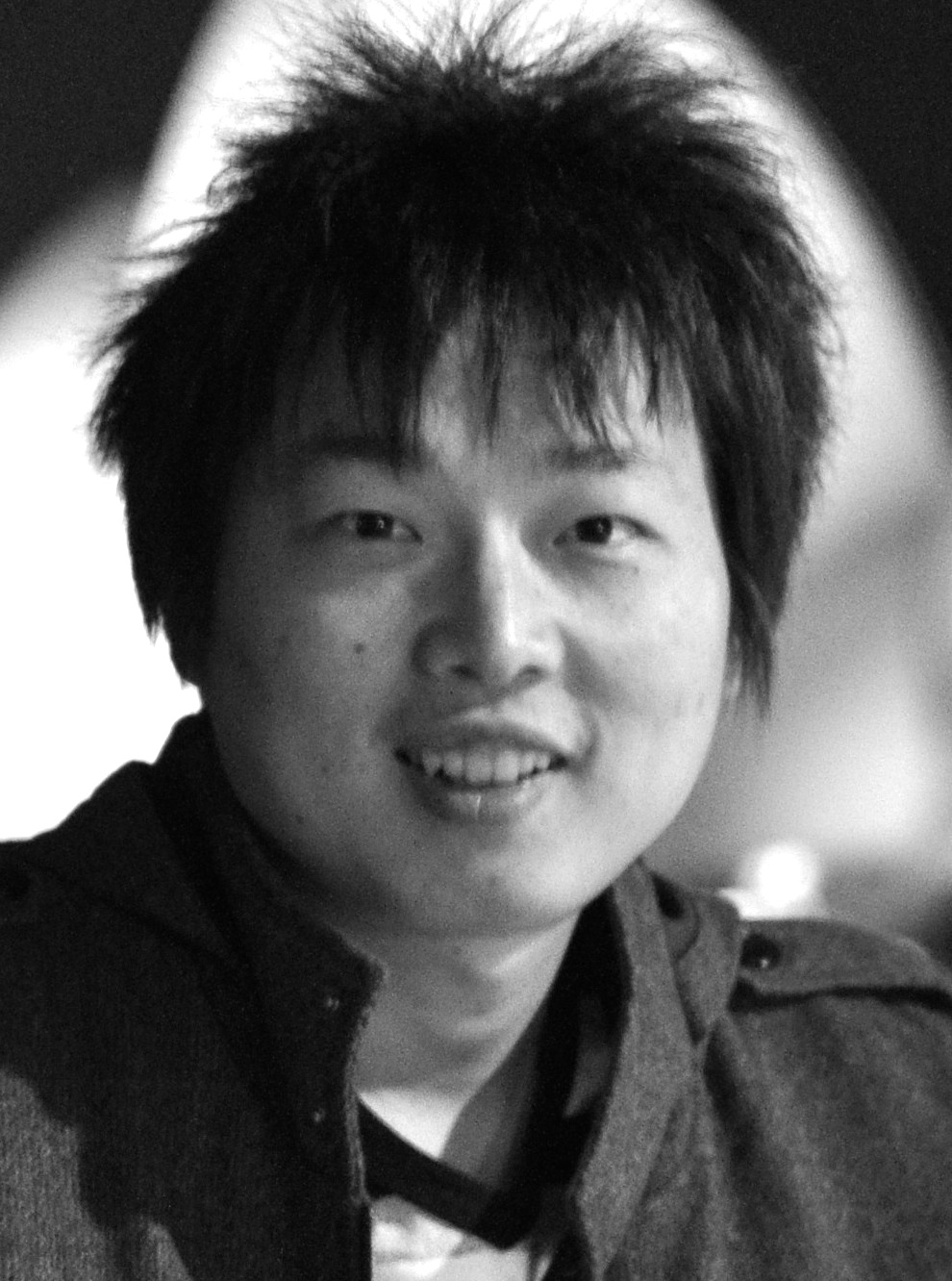}}]{Qinfeng Shi} is  a DECRA research fellow in The Australian Centre for Visual Technologies and the School of Computer Science, The University of Adelaide.  He received a PhD in computer science in 2011 at The Australian National University (ANU) after completing Bachelor and Master study in computer science and Technology in 2003 and 2006 at The Northwestern Polytechnical University (NPU).
\end{IEEEbiography}
\begin{IEEEbiography}[{\includegraphics[width=1in,height=1.25in,clip,keepaspectratio]{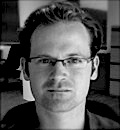}}]{Mark Reid}
is a Research Fellow at The Australian National University in Canberra. 
He received a PhD in machine learning in 2007 from the University of New South Wales
after completing a Bachelor of Science with honours in Pure Mathematics and Computer Science in 1996 from the same institution. In between, he worked as a research scientist at various companies including IBM and Canon.
\end{IEEEbiography}
\begin{IEEEbiography}[{\includegraphics[width=1in,height=1.25in,clip,keepaspectratio]{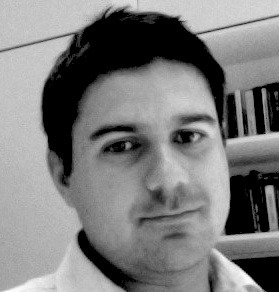}}]{Tiberio Caetano} received the BSc degree in
electrical engineering (with research in physics)
and the PhD degree in computer science, (with
highest distinction) from the Universidade Federal
do Rio Grande do Sul (UFRGS), Brazil. The
research part of the PhD program was undertaken
at the Computing Science Department at
the University of Alberta, Canada. He is a principal researcher with the Statistical Machine Learning
Group at NICTA, an adjunct senior fellow at the Research
School of Computer Science, Australian National
University, and a honorary researcher at the School of Information Technologies,
The University of Sydney.

\end{IEEEbiography}
\begin{IEEEbiography}[{\includegraphics[width=1in,height=1.25in,clip,keepaspectratio]{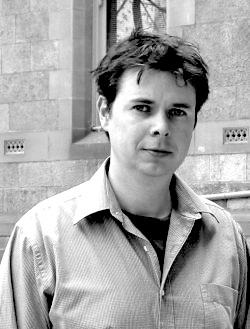}}]{Anton van~den~Hengel}
Prof van den Hengel is the founding Director of The Australian Centre for Visual Technologies (ACVT).  Prof van den Hengel received a PhD in Computer Vision in 2000, a Masters Degree in Computer Science in 1994, a Bachelor of Laws in 1993, and a Bachelor of Mathematical Science in 1991, all from The University of Adelaide.
 \end{IEEEbiography}
 \begin{IEEEbiography}[{\includegraphics[width=1in,height=1.25in,clip,keepaspectratio]{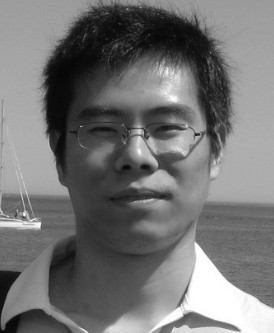}}]{Zhenhua Wang}
 is a Ph.D Candidate in The Australian Centre for Visual Technologies and the School of Computer Science, The University of Adelaide. He is supervised by  Prof. Anton van den Hengel, Dr. Qinfeng Shi and Dr. Anthony Dick. He received Bachelor's degree in 2007, and Master's degree in 2010, both from Northwest A\&F University.
  \end{IEEEbiography}


\vfill



%

\cleardoublepage \onecolumn

\end{document}